\documentclass{article}

\usepackage[final,nonatbib]{neurips_2022}

\usepackage[utf8]{inputenc} %
\usepackage[T1]{fontenc}    %
\usepackage{hyperref}       %
\usepackage{url}            %
\usepackage{booktabs}       %
\usepackage{amsfonts}       %
\usepackage{nicefrac}       %
\usepackage{microtype}      %
\usepackage{xcolor}         %
\usepackage{proof-at-the-end}  %

\title{Outsourcing Training without Uploading Data \\
via Efficient Collaborative Open-Source Sampling}
\usepackage{amsthm}
\usepackage{amssymb}
\usepackage{mathtools}
\usepackage{hyperref}
\usepackage[nameinlink,capitalize]{cleveref}
\hypersetup{colorlinks=true,linkcolor=blue,citecolor=blue,urlcolor=blue,pdfborder={0 0 0}}
\usepackage[normalem]{ulem} %
\usepackage{mathtools}

\usepackage[utf8]{inputenc} %
\usepackage[T1]{fontenc}    %
\usepackage{hyperref}       %
\usepackage{url}            %
\usepackage{booktabs}       %
\usepackage{amsfonts}       %
\usepackage{nicefrac}       %
\usepackage{microtype}      %
\usepackage{color, colortbl}
\definecolor{Gray}{gray}{0.9}
\newcolumntype{g}{>{\columncolor{Gray}}c}

\usepackage{graphicx,wrapfig}
\usepackage{algorithm}
\usepackage[noend]{algpseudocode}
\usepackage{caption}
\usepackage{subcaption}
\usepackage{multirow}
\usepackage{amsfonts}
\usepackage{url}
\usepackage{enumitem}
\usepackage{amsthm}
\usepackage{thmtools}
\usepackage{thm-restate}
\newtheorem{theorem}{Theorem}[section]

\newtheorem{lemma}{Lemma}[section]

\theoremstyle{definition}
\newtheorem{definition}{Definition}[section]
\theoremstyle{remark}

\usepackage{pifont}%
\newcommand{\cmark}{\ding{51}}%
\newcommand{\xmark}{\ding{55}}%

\newcommand{\xsm}[1]{\scalebox{.8}{#1}}
\definecolor{Gray}{gray}{0.93}
\newcolumntype{a}{>{\columncolor{Gray}}c}
\newcommand{\colorline}[1]{
\hspace{-0.01\linewidth}\colorbox{gray!20}{\makebox[0.99\linewidth][l]{#1}}
 }
\newcommand{\rev}[1]{#1}  %
\definecolor{lightred}{rgb}{0.847, 0.435, 0.435}

\newcommand{\cC}{{\mathcal{C}}}
\newcommand{\cD}{{\mathcal{D}}}

\newcommand{\cH}{{\mathcal{H}}}

\newcommand{\cM}{{\mathcal{M}}}
\newcommand{\cN}{{\mathcal{N}}}
\newcommand{\cO}{{\mathcal{O}}}

\newcommand{\cS}{{\mathcal{S}}}

\newcommand{\cX}{{\mathcal{X}}}
\newcommand{\cY}{{\mathcal{Y}}}

\newcommand{\NN}{\mathbb{N}}

\DeclareMathOperator*{\argmax}{arg\,max}

\newcommand{\bc}{\begin{center}}
\newcommand{\ec}{\end{center}}

\newcommand{\bdm}{\begin{displaymath}}
\newcommand{\edm}{\end{displaymath}}

\newcommand{\beq}{\begin{equation}}
\newcommand{\eeq}{\end{equation}}

\newcommand{\bfl}{\begin{flushleft}}
\newcommand{\efl}{\end{flushleft}}

\newcommand{\bt}{\begin{tabbing}}
\newcommand{\et}{\end{tabbing}}

\newcommand{\beqn}{\begin{align}}
\newcommand{\eeqn}{\end{align}}

\newcommand{\beqs}{\begin{align*}} %
\newcommand{\eeqs}{\end{align*}}  %

\newcommand{\norm}[1]{\left\|#1\right\|}

\newcommand{\etal}{\emph{et al.}}
\newcommand{\Ebb}{\mathbb{E}}

\author{%
	Junyuan Hong\thanks{Work done during internship at Sony AI. Corresponding to: Lingjuan Lyu.} \\
	Michigan State University \\ %
	\texttt{hongju12@msu.edu} \\
	\And
	Lingjuan Lyu \\
	Sony AI \\
	\texttt{lingjuan.lv@sony.com} \\
	\AND
	Jiayu Zhou \\
	Michigan State University \\ %
	\texttt{jiayuz@msu.edu} \\
	\And
	Michael Spranger \\
	Sony AI \\
	\texttt{michael.spranger@sony.com}
}

\begin{document}

\maketitle

\begin{abstract}
As deep learning blooms with growing demand for computation and data resources, outsourcing model training to a powerful cloud server becomes an attractive alternative to training at a low-power and cost-effective end device. 
Traditional outsourcing requires uploading device data to the cloud server, which can be infeasible in many real-world applications due to the often sensitive nature of the collected data and the limited communication bandwidth. 
To tackle these challenges, we propose to leverage widely available \emph{open-source data}, which is a massive dataset collected from public and heterogeneous sources (e.g., Internet images). 
We develop a novel strategy called Efficient Collaborative Open-source Sampling (ECOS) to construct a proximal proxy dataset from open-source data for cloud training, in lieu of client data. 
ECOS probes open-source data on the cloud server to sense the distribution of client data via a communication- and computation-efficient sampling process, which only communicates a few compressed public features and client scalar responses. 
Extensive empirical studies show that the proposed ECOS improves the quality of automated client labeling, model compression, and label outsourcing when applied in various learning scenarios. 
\end{abstract}

\section{Introduction}
\label{sec:intro}

Nowadays, powerful machine learning services are essential in many devices that supports our daily routines. Delivering such services is typically done through client devices that are power-efficient and thus very restricted in computing capacity. 
The client devices can collect data through built-in sensors and make predictions by machine learning models. However, their stringent computing power often makes the local training prohibitive, especially for high-capacity deep models. 
One widely adopted solution is to outsource the cumbersome training to cloud servers equipped with massive computational power, using machine-learning-as-a-service (MLaaS). Amazon Sagemaker~\cite{liberty2020elastic}, Google ML Engine~\cite{bisong2019google}, and Microsoft Azure ML Studio~\cite{barnes2015azure} are among the most successful industrial adoptions, where users upload training data to designated cloud storage, and the optimized machine learning engines then handle the training.  
One major challenge of the outsourcing solution in many applications is that the local data collected are sensitive and protected by regulations, therefore prohibiting data sharing. 
Notable examples include General Data Protection Regulation (GDPR)~\cite{gdpr} and Health Insurance Portability and Accountability Act (HIPPA)~\cite{act1996health}.

On the other hand, recent years witnessed a surging amount of general-purpose and massive datasets authorized for public use, such as ImageNet~\cite{deng2009imagenet}, CelebA~\cite{liu2015faceattributes}, and MIMIC~\cite{johnson2016mimic}.
Moreover, many task-specific datasets used by local clients can be well considered as biased subsets of these large public datasets~\cite{peng2019moment,li2020fedbn}. 
Therefore, the availability of these datasets allows us to use them to model confidential local data, facilitating training outsourcing without directly sharing the local data. 
One approach is to use the private client dataset to craft pseudo labels for a public dataset in a confidential manner~\cite{zhu2020privateknn,papernot2018scalable}, assuming that the public and local data are identically-and-independently-distributed (\emph{iid}). 
In addition, Alon~\etal ~showed that an \emph{iid} public data can strongly supplement client learning, which greatly reduces the private sample complexity~\cite{alon2019limits}.
However, the \emph{iid} assumption can often be too strong for general-purpose \emph{open-source} datasets, since they are usually collected from heterogeneous sources with distributional biases from varying environments. For example, a search of `digits' online yields digits images from handwriting scans, photos, to artwork of digits. 

\begin{figure*}
    \centering
    \includegraphics[width=\textwidth]{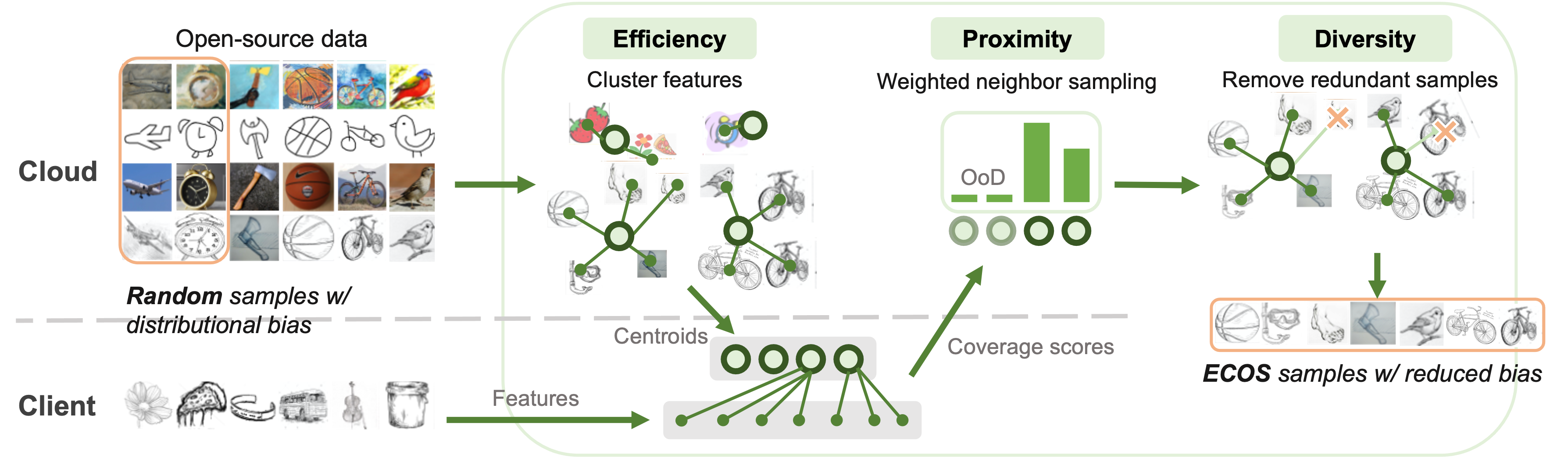}
    \caption{Illustration of the proposed ECOS framework.
    Instead of uploading local data for cloud training, ECOS downloads the centroids of clustered open-source features to \emph{efficiently} sense the client distribution, where the client counts the local neighbor samples of the centroids as the coverage score.
    Based on the the scores of centroids, the server adaptively samples \emph{proximal} and \emph{diverse} data for training a transferable model on the cloud.
    }
    \label{fig:priv_sample_main_idea}
    \vspace{-0.2in}
\end{figure*}

In this paper, we relax the \emph{iid} assumption in training outsourcing and instead consider the availability of an open-source dataset.
We first study the gap between the \emph{iid} data and the heterogeneous open-source data in training outsourcing, and show the low sample efficiency of open-source data.
We show that in order to effectively train a model from open-source data that is transferable to the client data, %
the open-source data needs to communicate more samples %
than those of \emph{iid} data. 
The main reason behind such low sample efficiency is that we accidentally included out-of-distribution (OoD) samples, which poison the training and significantly degrade accuracy at the target (client) data distribution~\cite{biggio2012poisoning}.
We propose a novel framework called Efficient Collaborative Open-source Sampling (ECOS) to tackle this challenge, which filters the open-source dataset through an efficient collaboration between the client and server and does not require client data to be shared. During the collaboration, the server sends compressed representative features (centroids) of the open-source dataset to the client.
The client then identifies and excludes OoD centroids and returns their privately computed categorical scores to the server.
The server then adaptively and diversely decompresses the neighbors of the selected centroids. The main idea is illustrated in \cref{fig:priv_sample_main_idea}.

Our major contributions are summarized as follows: \\
    $\bullet$ \textit{New problem and insight}: Motivated by the strong demands for efficient and confidential outsourcing, using public data in place of the client data is an attractive solution. However, the impact of heterogeneous sources of the public data, namely open-source data, is rarely studied in existing works. Our empirical study shows the potential challenges due to such heterogeneity.\\
    $\bullet$ \textit{New sampling paradigm}: We propose a new unified sampling paradigm, where the server only sends very few query data to the client and requests very few responses that efficiently and privately guide the cloud for various learning settings on open-source data.
    To our best knowledge, our method enables efficient cloud outsourcing under the most practical assumption of open-source public data, and does not require accessing raw client data or executing cumbersome local training. \\
    $\bullet$ \textit{Compelling results}: In all three practical learning scenarios, our method improves the model accuracy with pseudo, manual or pre-trained supervisions. Besides, our method shows competitive efficiency in terms of both communication and computation.
\section{Related Work} %
\label{sec:related}
\vspace{-0.1in}
There are a series of efforts studying how to leverage the data and computation resources on the cloud to assist client model training, especially when client data cannot be shared \cite{yao2021edgecloud,wang2019edge}. We categorize them as follows:
\textit{1)~Feature sharing}: Methods like group knowledge transfer~\cite{he2020group}, split learning~\cite{vepakomma2018split} and domain adaptation~\cite{dong2020what,Dong2021transfer} transfer edge knowledge by communicating features extracted by networks. To provide a theoretical guarantee of privacy protection, \cite{osia2020deep} proposed an advanced information removal to disentangle sensitive attributes from shared features.  %
In the notion of rigorous privacy definition, Liu~\etal\ leveraged public data to assist private information release~\cite{liu2021leveraging}.
Earlier, data encryption was used for outsourcing, which however is too computation-intensive for a client and less applicable for large-scale data and deep networks~\cite{chen2014highly,lei2013outsourcing}.
\rev{Federated Learning (FL)~\cite{mcmahan2017communication} considers the same constraint on data sharing but allocates the burdens of training~\cite{hong2022efficient} and communication~\cite{zhu2022resilient} to clients and opens up a series of challenges on privacy~\cite{chen2022practical}, security~\cite{lyu2020privacy,chen2022calfat} and knowledge transfer~\cite{hong2021federated}.}
\textit{2)~Private labeling:} 
\rev{PATE and its variants were proposed to generate client-approximated labels for unlabeled public data, on which a model can be trained~\cite{papernot2016semi,papernot2018scalable}.
Without training multiple models by clients, Private kNN was a more efficient alternative which explored the private neighborhood of public images for labeling~\cite{zhu2020privateknn}.}
These approaches are based on a strong assumption of the availability of public data that is \emph{iid} as the local data.
This paper considers a more practical yet challenging setting where public data are from multiple agnostic sources with heterogeneous features.

Sampling from public data has been explored in central settings.
For example, Xu~\etal~\cite{xu2019positiveunlabeled} used a few target-domain samples as a seed dataset to filter the open-domain datasets by positive-unlabeled learning~\cite{loghmani2020positiveunlabeled}.
\rev{Yan~\etal~\cite{yan2020neural} used a model to find the proxy datasets from multiple candidate datasets.}
In self-supervised contrastive learning, model-aware $K$-center (MAK) used a model pre-trained on the seed dataset to find desired-class samples from open-world dataset~\cite{jiang2021improving}.
Though these methods provided effective sampling, they are less applicable when the seed dataset is placed at the low-energy edge, because the private seed data at the edge cannot be shared with the cloud for filtering and the edge device is incapable of computation-intensive training.
To address these challenges, we develop a new sampling strategy requiring only light-weight computation at the edge.
\section{Outsourcing Model Training With Open-Source Data}
\label{sec:method}
\vspace{-0.1in}

\subsection{Problem Setting and Challenges}

\begin{table}[b]
    \caption{Test accuracy (\%) with different client domains (columns). Cloud data are identically distributed as the client data (\textit{ID}) or including more data from 5 distinct domains (\textit{ID+OoD}) without overlapped samples. We first label a number of randomly selected cloud examples (i.e., sampling budget) privately by client data~\cite{zhu2020privateknn}, and then train a classifier to recognize digit images. The privacy cost $\epsilon$ is accounted for in the notion of differential privacy. Larger budgets imply more privacy and communication costs.
    More results on different settings are enclosed in \cref{sec:app:smp_eff}.
    }
    \label{tbl:ood_ablation}
    \centering
    \setlength\tabcolsep{2 pt}
    \scriptsize
    \begin{tabular}{cc*{10}{c}ac}
        \toprule
        \rowcolor{white}
        Cloud & Sampling & \multicolumn{2}{c}{MNIST} & \multicolumn{2}{c}{SVHN} & \multicolumn{2}{c}{USPS} & \multicolumn{2}{c}{SynthDigits} & \multicolumn{2}{c}{MNIST-M} & \multicolumn{2}{c}{Average} \\
        \rowcolor{white}
        Data & Budget & Acc (\%) $\uparrow$ & $\epsilon \downarrow$  & Acc (\%) $\uparrow$ & $\epsilon \downarrow$  & Acc (\%) $\uparrow$ & $\epsilon \downarrow$  & Acc (\%) $\uparrow$ & $\epsilon \downarrow$  & Acc (\%) $\uparrow$ & $\epsilon \downarrow$ & Acc (\%) $\uparrow$ & $\epsilon \downarrow$ \\
        \midrule
        ID & 1000 & $\mathbf{84.3}$\xsm{$\pm 2.4$} &  4.48 & $\mathbf{51.6}$\xsm{$\pm 1.4$} &  4.08 & $\mathbf{87.1}$\xsm{$\pm 0.5$} &  4.51 & $\mathbf{73.2}$\xsm{$\pm 1.5$} &  4.57 & $\mathbf{55.5}$\xsm{$\pm 1.0$} &  4.46 & $\mathbf{70.4}$ & 4.42 \\
        \midrule
        \multirow{4}{*}{ID+OoD} & 1000 & $78.0$\xsm{$\pm 3.5$} &  4.30 & $40.6$\xsm{$\pm 1.6$} &  3.75 & $82.2$\xsm{$\pm 2.7$} &  4.32 & $62.1$\xsm{$\pm 1.6$} &  4.41 & $49.1$\xsm{$\pm 1.0$} &  4.27 & $62.4$ & 4.21 \\
        & 8000 & $82.2$\xsm{$\pm 4.1$} &  5.89 & $47.9$\xsm{$\pm 1.8$} &  5.89 & $85.4$\xsm{$\pm 0.5$} &  5.89 & $64.4$\xsm{$\pm 3.6$} &  5.89 & $53.3$\xsm{$\pm 2.2$} &  5.89 & $66.6$ & 5.89 \\
        & 16000 & $82.6$\xsm{$\pm 1.4$} &  7.17 & $48.5$\xsm{$\pm 1.7$} &  7.17 & $86.7$\xsm{$\pm 1.9$} &  7.17 & $67.5$\xsm{$\pm 2.3$} &  7.17 & $52.0$\xsm{$\pm 3.0$} & 7.17 & $67.4$ & 7.17 \\
        & 32000 & $84.1$\xsm{$\pm 1.6$} &  9.32 & $49.4$\xsm{$\pm 0.2$} &  9.32 & $86.8$\xsm{$\pm 2.0$} &  9.32 & $68.5$\xsm{$\pm 0.1$} &  9.32 & $53.0$\xsm{$\pm 2.7$} &  9.32 & $68.4$ & 9.32 \\
        \bottomrule
    \end{tabular}
\end{table}

Motivated in \cref{sec:intro}, we aim to outsource the training process from computation-constrained devices to the powerful cloud server, where a proxy public dataset without privacy concerns is used in place of the client dataset for cloud training.
One solution is (private) client labeling by k-nearest-neighbors (kNN) \cite{zhu2020privateknn}, where the client and cloud server communicate the pseudo-label of a public dataset privately and the server trains a classifier by the labeled and unlabeled samples in a semi-supervised manner.
The success of this strategy depends on the key assumptions that public data in the cloud and private data in the client are \emph{iid}, which are rather strong in practice and thus prevent it from many real-world applications. 
In this work, we make a more \emph{realistic} assumption that the public datasets are as accessible as \emph{open-source} data.
An open-source dataset consists of biased features from multiple heterogeneous sources (feature domains), and therefore includes not only in-distribution (ID) samples similar to the client data but also multi-domain OoD samples.

The immediate question is how the OoD samples affect the outsourced training.
In \cref{tbl:ood_ablation}, we empirically study the problem by using a 5-domain dataset, Digits, where 50\% of one domain is used on the client and the remained 50\% together with the other 4 domains serve as the public dataset on the cloud.
To conduct the cloud training, we leverage the client data to generate pseudo labels for the unlabeled public samples, following \cite{zhu2020privateknn}.
It turns out that the presence of OoD samples in the cloud greatly degrades the model accuracy.
The inherent reason for the degradation is that the distributional shift of data \cite{quinonero2008dataset} compromised the transferability of the model to the client data~\cite{weiss2016survey}.

\textbf{Problem formulation by sampling principles}. Given a client dataset $D^p$ and an open-source dataset $D^q$, the goal of open-source sampling is to find a proper subset $S$ from $D^q$, whose distribution matches $D^p$.
In~\cite{jiang2021improving}, Model-Aware K-center (MAK) formulated the sampling as a principled optimization:
\begin{align}
    \min\nolimits_{S\subseteq D^q} %
    \Delta(S, D^p) - H(S\cup D^p; D^q), \label{eq:principle_obj}
\end{align}
where %
\rev{$\Delta(S, D^p):=\Ebb_{x'\in D^p} [ \min_{x\in S} \norm{\phi(x)-\phi(x')}^2 ]$} measures \emph{proximity} as the set difference between $S$ and $D^p$ using a feature extractor $\phi$, and the latter \rev{$H(S\cup D^p;D^q) := \max_{x'\in D^q} \min_{x\in S\cup D^p} \norm{\phi(x)-\phi(x')}^2 $} measures \emph{diversity} by contradicting $S\cup D^p$ and $D^q$ (suppose $D^q$ is the most diverse set)\footnote{Note that we use $L_2$-norm distance instead of normalized cosine similarity in $\Delta(S, D^p)$ in contrast to MAK, since normalized cosine similarity is not essential if the feature space is not trained under the cosine metric. We also omit the tailedness objective which is irrelevant in our context.}. %
Solving \cref{eq:principle_obj} results in an NP-hard problem that is intractable~\cite{cook2009combinatorial}, 
and MAK provides an approximated solution by a coordinate-wise greedy strategy.
It first pre-trains the model representations on $D^p$ and finds a large candidate set with the best proximity to extracted features.
Then, it incrementally selects the most diverse samples from the candidate set until the sampling budget is used up.

Though MAK is successful in the central setting, it is not applicable when $D^p$ is isolated from cloud open-source data and is located at a resource-constrained client for two reasons:
1)~\textit{Communication inefficiency}. Uploading client data may result in privacy leakage, sending public data to the client is a direct alternative but the cost can be prohibitive.
2)~\textit{Computation inefficiency}. Pre-training a model on $D^p$ or proximal sampling (which computes the distances between paired samples from $D^q$ and $D^p$) induces unaffordable computation overheads for the low-energy client.

\subsection{Proposed Solution: Efficient Collaborative Open-Source Sampling (ECOS)}

To address the above challenges, we design a new strategy that 1) uses compressed queries to reduce the communication and computation overhead
and 2) uses a novel principled objective to effectively \emph{sample} from open-source data with the client responses of the compressed queries.

\textbf{Construct communication-efficient and an informative query set $\hat \Phi^q$ at cloud}.
Let $d$ be the number of pixels of an image, the communication overhead of transmitting $D^q$ to the client is given by $\cO(d |D^q|)$.
For communication efficiency, we optimize the following two factors:\\
i) \textit{Data dimension $d$.} First, we transmit extracted features $\Phi^q=\{\phi(x) | x\in D^q\}$ instead of images to reduce the communication overhead to $\cO(d_e |D^q|)$, where $d_e$ is a much smaller embedding dimension.
For accurate estimation of the distance $\Delta$, a pre-defined discriminative feature space is essential without extra training on the client.
Depending on resources, one may consider hand-crafted features such as HOG \cite{dalal2005histograms}, or an off-the-shelf pre-trained model such as ResNet pre-trained on ImageNet.\\
ii) \textit{Data size $|D^q|$.} Even with compression, sending all data for querying is inefficient due to the huge size of open-source data $|D^q|$. 
Meanwhile, too many queries would cast unacceptable privacy costs to the client.
As querying on similar samples leads to redundant information in querying, we propose to reduce such redundancy by selecting informative samples.
We use the classic clustering method KMeans~\cite{forgy1965cluster} for compressing similar samples by clustering them, and collect the $R$ mean vectors or \emph{centroids} into $\hat \Phi^q = \{c_r\}_{r=1}^R$. We denote $R$ as the \emph{compression size} and $\hat D^q$ as the set of original samples corresponding to $\hat \Phi^q$.

\textbf{New sampling objective}. We note that sending the compact set $\hat \Phi^q$ in place of $D^q$ prohibits the client from optimizing $\Delta(S, D^p)$ in \cref{eq:principle_obj} for $S \in D^q$. %
\rev{Instead, we sample a set of centroids  $\hat S \in \hat \Phi^q$ and decompress them by the cluster assignment into corresponding original samples with rich features afterwards.
In principle, we leverage the inequality $\Delta(S, D^p) \le \Delta(\hat S, D^p) + \Delta(\hat S, S)$ to attain a communication-efficient surrogate objective as follows:}
\begin{align}
    \min \nolimits_{\hat S \subseteq \hat D^q, S\subseteq D^q} \ 
    \underbrace{\Delta(\hat S, D^p) + \Delta(\hat S, S)}_{\text{proximity}} 
    - \underbrace{H(S; D^q)}_{\text{diversity}},
    \label{eq:ecos_obj}
\end{align}
where $\hat S$ (or $\hat D^q$) is the compressed centroid substitute of $S$ (or $D^q$).
Different from \cref{eq:principle_obj}, we decompose the proximity term into two in order to facilitate communication efficiency leveraging an informative subset $\hat D^q$.
We solve the optimization problem in a greedy manner by two steps at the client and the cloud, respectively: \\
i) At the \emph{client step}, we optimize $\Delta(\hat S, D^p)$ to find a subset of centroids ($\hat S \subset \hat D^q$) that are proximal to the client set $D^p$. %
Noticing that $\hat D^q$ contains the cluster centroids, we take advantage of the property to define a novel proximity measure of the cluster $r$: Centroid Coverage (CC), denoted as $v_r$.
Upon receiving centroids from the cloud, the client uses them to partition the local data into $\{\cC^p_r\}_{r=1}^R$ where $\cC^p_r$ denotes the $r$-th cluster partition of local data. We compute the CC score by the cardinality of the neighbor samples of the centroid $r$, i.e., $v_r=|\cC^p_r|$. 
To augment the sensitivity to the proximal clusters, we scale the CC score by a non-linear function $v'_r = \psi_s (|\cC^p_r|)$, where the scale function $\psi_s(x) = x^s$ is parameterized by $s$.\\
ii) At the \emph{cloud step}, we optimize the proximity of $S$ w.r.t. the proxy set $\hat S$, i.e., $\Delta(\hat S, S)$, and remove redundant and irrelevant samples from the candidate set to encourage diversity, i.e., $- H(S; D^q)$.
As samples among clusters are already diversified by KMeans, we only need to promote the in-cluster \emph{diversity}.
To this end, we reduce the sample redundancy per cluster at cloud by K-Center \cite{sener2018active}, which heuristically finds the most diverse samples.
Such design transfers the diversity operation to cloud and thus reduces the local computation overhead.
To maintain the \emph{proximity}, the K-Center is applied within each cloud cluster and the sampling budget per cluster is proportional to their vote numbers and the original cluster sizes.
With the normalized scores, we compute the sampling budget per cluster which is upper bounded by the ratio of the cluster in the cloud set.

\begin{algorithm}
\caption{Efficient collaborative open-source sampling (ECOS)
}\label{alg:ours}
\begin{algorithmic}[1]
    \Require Client dataset $D^p$, cloud query dataset $D^q$, sampling budget $B$, compression size $R$, feature extractor $\phi(\cdot)$, distance function $\Delta(x, S)=\min_{y\in S} \norm{ \phi(x)- \phi(y)}$, initial sample set $S=\emptyset$, score scale function $\psi_s(x)=x^s$.
    \State Extract features $\Phi^q = \{\phi(x) | x\in D^q\}$;
    \State Cloud creates a compressed dataset $\hat \Phi^q= \text{KMeans}_R(\Phi^q)$; \Comment{Compress $R$ Centroids}
    \State \colorline{$\triangleright \triangleright\triangleright$ \textbf{Client End} $\triangleright \triangleright\triangleright$}
    \State Download the feature extractor $\phi$ and $\hat \Phi^q$;
    \State Use centroids $\hat \Phi^q$ to partition $\Phi^p = \{\phi(x) | x\in D^p\}$ into clusters $\{\cC_{r}^p\}_{r=1}^R$;
    \State Compute the Centroid Coverage (CC) scores: $v_r=|\cC_{r}^p|, ~ \forall r\in \{1,\dots,R\}$;
    \State Upload scaled cluster scores $\{v_r' = \psi_s(v_r)\}_{r=1}^R$; \Comment{Proximity}
    \State \colorline{$\triangleleft \triangleleft \triangleleft$ \textbf{Cloud End} $\triangleleft \triangleleft \triangleleft$}
    \State Partition $D^q$ into clusters $\{\cC_{r}^q\}_{r=1}^R$ by centroids $\hat \Phi^q$;
    \State Compute per-cluster sampling budget $b_r = \min\left\{ \frac{|\cC_r^q|}{\sum_j |\cC_j^q|}, \frac{v_r'}{\sum_j v_j'} \right\}\cdot B$;
    \For{$r$ in $\{1, \dots, R\}$} \Comment{Decompress Centroids}
        \State Initialize $S' = \{x\}$ by a sample randomly picked from $\cC_r$;
        \While{$|S'| < b_r$} \Comment{Diverse Sampling}
        \State $u = \argmax_{x\in \cC^q_r} \Delta(x, S')$;
        \State $S' = \{u\} \cup S'$;
        \EndWhile
        \State $S=S' \cup S$;
    \EndFor
    \State \Return $S$
\end{algorithmic}
\end{algorithm}

We summarize our algorithm in \cref{alg:ours}, which enables the clients to enjoy better computation efficiency than local training and better communication efficiency than the centralized sampling (e.g., MAK).
\textbf{1)~Computation efficiency}.
Since our method only requires inference operations on the client device, which should be efficiently designed for the standard predictive functions of the device, and is training-free for the client, the major complexity of ECOS is on computing centroid coverage and is much lower than gradient-based algorithms whose complexity scales with the model size and training iterations.
As computing the CC scores only requires the nearest centroid estimation and ranking, the filtering can be efficiently done.
The total \textit{time complexity} is $\cO(C_{\phi} |D^p| + (d_e+1) R |D^p|)$, dominated by the first term, where $C_{\phi}$ is the complexity of extracting features depending on the specific method.
The second term $(d_e+1) R |D^p|$ is for computing the pair-wise distances between $\Phi^p$ and $\hat \Phi^q$ and estimating the nearest centroids per sample (or partitioning client data).
In a brief comparison, the complexity of local $T$-iteration gradient-descent training could be approximately $\cO(TC_{\phi}|D^p|)$ which is much more expensive since typically $TC_{\phi} \gg d_e$.
To complete the analysis, the \textit{space complexity} is $\cO(C_{\phi}' + (d+d_e) |D^p| + d_e R + |D^p| R)$ for the memory footprint of $\phi$, the images ($d$) or features ($d_e$) of client and cloud centroid data, and the distance matrix.
\textbf{2)~Communication efficiency}.
The downloading complexity will be $\cO(d_e R)$ for $R$ $d_e$-dimensional centroid features and the uploading complexity is $\cO(R)$ including indexes of samples.
Thus, the data that will be communicated between the client and cloud is approximately of $\cO(d_e R + R)$ complexity in total.
In comparison, downloading the whole open-source dataset by central sampling (e.g., MAK) requires $\cO(d |D^q| + B)$ complexity.
As $R < d_e R \ll d_e |D^q| \ll d |D^q|$, our method can significantly scale down the computation cost.

\textbf{Privacy protection and accountant}.
When the cloud server is compromised by an attacker, uploading CC scores leak private information of the local data samples, for example, the presence of an identity~\cite{shokri2017membership}.
To mitigate the privacy risk, we protect the uploaded scores by a Gaussian noise mechanism, i.e., $\tilde v_r = v_r + \cN(0, \sigma^2)$, and account for the privacy cost in the notion of differential privacy (DP)~\cite{dwork2006calibrating}.
DP quantifies the numerical influence of the absence of a private sample on $[v_1, \cdots, v_R]$, which is connected to the chance of exposing the sample to the attacker.
To obtain a tight bound on the privacy cost, we utilize the tool of R\'enyi Differential Privacy (RDP)~\cite{mironov2017renyi} and leverage the Poisson sampling to further amplify the privacy~\cite{zhu2019poission}.
With the noise mechanism governed by $\sigma$, the resultant privacy cost in the sense of $(\epsilon, \delta)$-DP can be accounted as $\epsilon = \cO(\gamma \sqrt{\log(1/\delta)} /{\sigma})$ where $\gamma$ is the Poisson subsampling rate and $\delta$ is a user-specified parameter.
A larger $\epsilon$ implies higher risks of privacy leakage in the probability of $\delta$.
Formal proofs can be found in \cref{sec:app:privacy_proof}.

\rev{
\textbf{Generalization error of models trained on cloud.} 
Our work can be viewed as knowledge transfer from the open-source domain to the private client domain. 
Therefore, we present \cref{thm:error_ecos} based on prior domain-adaptation theoretical results~\cite{blitzer2007learning,dong2020what,Dong2021transfer}.
\begin{theorem}\label{thm:error_ecos}
Assume that a open-source dataset $S$ is induced from a mixture of cluster distributions, i.e., $\sum_{r=1}^R \alpha_r \cD_r^q$ with cluster distribution $\cD^q_r$, $\alpha_r\in[0,1]$ and $\sum_{r=1}^R \alpha_r =1$.
Suppose client data are sampled from $\cD^p$.
Let $L(\cdot, \cdot)$ be a loss function on a hypothesis and a dataset (for empirical error) or a distribution (for generalization error).
If $f$ is governed by the parameter $\theta$ trained on $S$ and belongs to a hypothesis space $\cH$ of $VC$-dimension $d$,
then with probability at least $1 - p$ over the choice of samples, 
the inequality holds,
    \begin{align}
        L(f_{\theta}, \cD^p) &\le 
        L(f_{\theta}, S) %
        + \frac{1}{2} \sum\nolimits_{r=1}^R \alpha_r d_{\cH \Delta \cH} (\cD^q_{r}, \cD^p) + 4 \sqrt{\frac{2d \log (2 |S|) + \log(4/p)}{ |S|}} + \xi, \label{eq:main_bd_0}
    \end{align}
    where $\xi = \min_\theta \left\{ L(f_\theta, S) + L(f_\theta, D^p) \right\}$, and $d_{\cH \Delta \cH}(\cD, \cD')$ denotes the distribution divergence.
\end{theorem}
The proof of \cref{thm:error_ecos} is deferred to \cref{sec:app:thm}.
\cref{thm:error_ecos} shows that given a model trained on cloud, its generalization error on client data hinges on the quality and size of the ECOS-sampled subset.
Informally, if the sampled set has enough samples (the 3rd term) and follows a similar distribution as the client dataset (the 2nd term), then the model generalizes better via only training on the proxy cloud dataset (the 1st term).
In \cref{thm:error_ecos}, the model $f_\theta$ can be trained by any task-specific optimization method.
For example, minimizing a $\mu$-strongly-convex and $G$-smooth loss function $L(f_\theta, S)$ w.r.t. $\theta$ via $T$-iteration gradient descent leads to $L(f_{\theta}, S) = (1-\mu/G)^{T} | L(f_{\theta_0} , S) - L^* |$ where $L^*$ is the optimal loss.

Given \cref{eq:main_bd_0}, it can be shown that minimizing the upper bound requires a dedicated trade-off between sample size and quality.
Consider a simple case, only minimizing $\sum\nolimits_{r=1}^R \alpha_r d_{\cH \Delta \cH} (\cD^q_{r}, \cD^p)$.
The solution is that $\alpha_r$ equals $1$ if $d_{\cH \Delta \cH} (\cD^q_{r}, \cD^p)$ is smallest among all choices of clusters, otherwise zero.
However, the number of samples in the cluster $r$ is limited as $|\cC_r^q|$ which is much smaller than the whole open-source dataset $|D^q|$, and therefore $|S|$ will be too small to enlarge the third term in \cref{eq:main_bd_0}.
Thus, $\alpha_r$ should be smoothed to increase sample size by trading in some quality (distribution divergence).
In our implementation, we approximate $\alpha_r$ via the CC score $v_r^s$, where an $s<\infty$ smooths the $\alpha_r$ to trade off sample size and quality.
When $s$ vanishes, all clusters will be selected at the same chance and the divergence could be large.
}

\section{Empirical Results}
\label{sec:exp}

\textbf{Datasets.}
We use datasets from two tasks: digit recognition and object recognition.
Distinct from prior work \cite{zhu2020privateknn}, in our work, the open-source data contains samples out of the client's distribution.
With the same classes as the client dataset, we assume open-source data are from different environments and therefore include different feature distributions, for example, DomainNet~\cite{peng2019moment} and Digits~\cite{li2020fedbn}.
DomainNet includes large-sized $244\times 244$ everyday images on 6 domains and Digits consists of $28\times 28$ digit images on 5 domains.
Instead of using an overly large volume of data from a single domain like~\cite{zhu2020privateknn}, we tailor each-domain subset to contain fewer images than standard digit datasets, for instance, MNIST with 7438 images, which was previously adopted in the distributed learning setting~\cite{li2020fedbn} and mitigates the hardness of collecting enormous data.
In practice, collecting tens of thousands of images by a single client, e.g., 50000 images from MNIST domain, will be unrealistic.
Similarly, DomainNet will be tailored to only include 10 classes with 2000-5000 images per domain.
\textbf{Splits of client and cloud datasets}. For Digits, we use one domain for the client and the rest domains for the cloud as open-source set. For DomainNet, we randomly select 50\% samples from one domain for the client and leave the rest samples together with all other domains to the cloud. The difference of configurations for the two datasets is caused by their different domain gaps. Even without ID data, it is possible for Digits to transfer the knowledge across domains. \\
\textbf{Baselines.}
For a fair comparison, we compare our method to baselines with the same sampling budget.
Each experiment case is repeated for three times with seeds $\{1,2,3\}$.
We account for the privacy cost by Poisson-subsampling RDP~\cite{zhu2019poission} and translate the cost to the general privacy notion, $(\epsilon, \delta)$-DP when $\delta=10^{-5}$.
Here, we use the \emph{random} sampling as a naive baseline.
We also adopt a coreset selection method, K-Center~\cite{sener2018active}, to select informative samples within the limited budget.
\rev{Both baselines are perfectly private, because they do not access private information from clients.}
Details of hyper-parameters are deferred to \cref{sec:exp_details}.

\subsection{Evaluations on Training Outsourcing}

To demonstrate the general applicability of ECOS, we present three practical use cases of outsourcing, categorized by the form of supervisions.
The conceptual illustrations are in \cref{fig:applications}.
Per use case, we train a model on partially labeled cloud data (outsourced training) and accuracy on the client test set is reported together with standard deviations.
We present the results in \cref{tbl:alabel,tbl:adaptation,tbl:test_acc_closed_set} case by case, where we vary the domain of the client by columns.
In each column, we highlight the best result unless the difference is not statistically significant.

\begin{figure*}[t]
    \centering
    \begin{subfigure}[b]{0.55\textwidth}
         \centering
         \includegraphics[width=\textwidth]{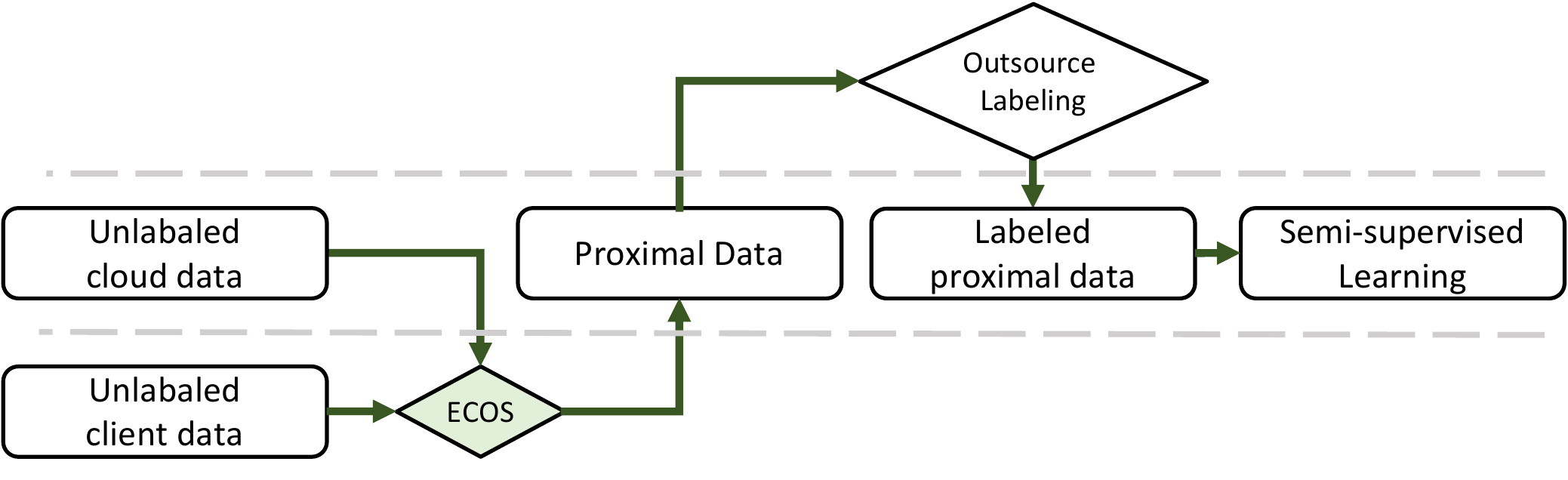}
         \vspace{-0.1in}
         \caption{Selective manual labeling.}
         \label{fig:app_outsource}
    \end{subfigure}
    \hfill
    \begin{subfigure}[b]{0.4\textwidth}
         \centering
         \includegraphics[width=\textwidth]{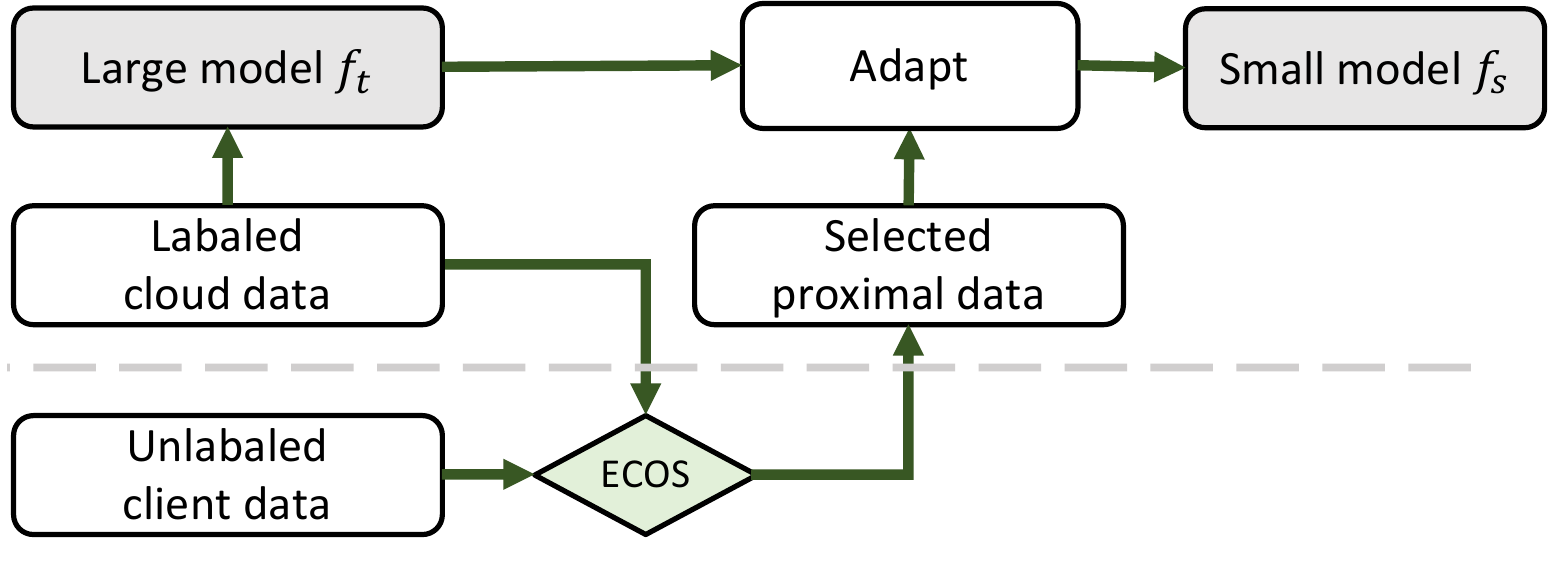}
         \vspace{-0.1in}
         \caption{Adaptive model compression.}
         \label{fig:app_comp}
    \end{subfigure}
    \hfill
    \begin{subfigure}[b]{0.6\textwidth}
         \centering
         \includegraphics[width=\textwidth]{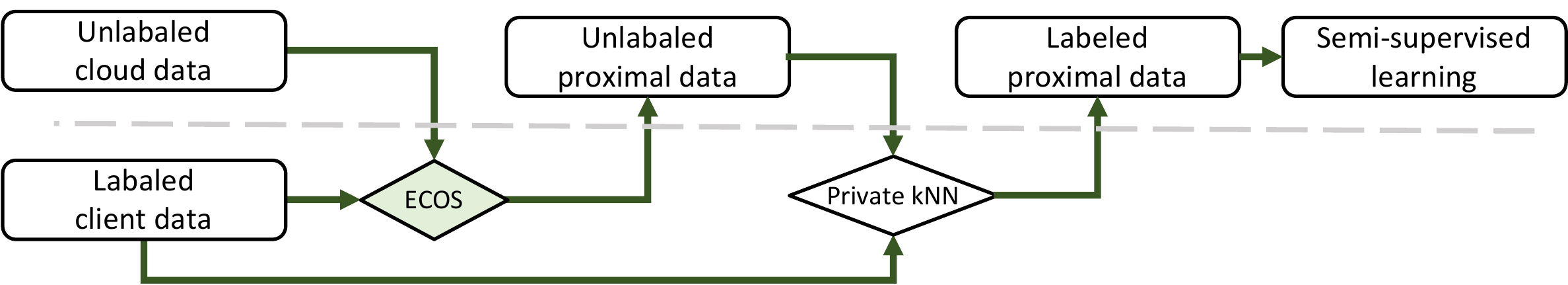}
         \vspace{-0.05in}
         \caption{Automated client labeling.}
         \label{fig:app_client_sup}
    \end{subfigure}
    \vspace{-0.1in}
    \caption{Our method is applicable to various cloud training cases, where ECOS filters the open-source samples to improve the model performance trained on (a) manual, (b) pre-trained model (teacher), and (c) pseudo supervisions.
    }
    \label{fig:applications}
\end{figure*}

\begin{table}[t]
    \caption{Test accuracy by selective labeling on Digits (top) and DomainNet (bottom).}
    \label{tbl:alabel}
    \centering
    \setlength\tabcolsep{2 pt}
    \scriptsize
    \begin{tabular}{cc*{10}{c}a}
        \toprule
        \rowcolor{white}
        \multicolumn{2}{c}{Sampling} & \multicolumn{2}{c}{MNIST} & \multicolumn{2}{c}{SVHN} & \multicolumn{2}{c}{USPS} & \multicolumn{2}{c}{SynthDigits} & \multicolumn{2}{c}{MNIST-M} & Average \\
        \rowcolor{white}
        Budget & Method & Acc (\%) $\uparrow$ & $\epsilon \downarrow$  & Acc (\%) $\uparrow$ & $\epsilon \downarrow$  & Acc (\%) $\uparrow$ & $\epsilon \downarrow$  & Acc (\%) $\uparrow$ & $\epsilon \downarrow$  & Acc (\%) $\uparrow$ & $\epsilon \downarrow$  & Acc (\%) $\uparrow$ \\
        \midrule
        \multirow{3}{*}{ 2000}  & Ours & $\mathbf{97.3}$\xsm{$\pm 0.1$} &  0.22 & $\mathbf{68.7}$\xsm{$\pm 0.3$} &  0.22 & $90.8$\xsm{$\pm 0.1$} &  0.22 & $\mathbf{84.4}$\xsm{$\pm 0.6$} &  0.22 & $70.4$\xsm{$\pm 0.6$} &  0.22 & $\mathbf{82.3}$ \\
         & K-Center & $96.7$\xsm{$\pm 0.3$} &  0.00 & $65.1$\xsm{$\pm 1.3$} &  0.00 & $90.1$\xsm{$\pm 0.7$} &  0.00 & $80.2$\xsm{$\pm 1.1$} &  0.00 & $70.1$\xsm{$\pm 0.3$} &  0.00 & $80.4$ \\
         & Random & $96.5$\xsm{$\pm 0.3$} &  0.00 & $64.0$\xsm{$\pm 0.3$} &  0.00 & $91.6$\xsm{$\pm 1.0$} &  0.00 & $83.8$\xsm{$\pm 0.3$} &  0.00 & $70.9$\xsm{$\pm 0.6$} &  0.00 & $81.4$ \\
        \midrule
        \multirow{3}{*}{ 5000}  & Ours & $\mathbf{98.1}$\xsm{$\pm 0.2$} &  0.22 & $\mathbf{74.6}$\xsm{$\pm 1.0$} &  0.22 & $\mathbf{93.5}$\xsm{$\pm 0.3$} &  0.22 & $\mathbf{91.2}$\xsm{$\pm 0.4$} &  0.22 & $74.5$\xsm{$\pm 0.5$} &  0.22 & $\mathbf{86.4}$ \\
         & K-Center & $97.9$\xsm{$\pm 0.2$} &  0.00 & $72.3$\xsm{$\pm 0.6$} &  0.00 & $92.7$\xsm{$\pm 0.9$} &  0.00 & $89.6$\xsm{$\pm 0.3$} &  0.00 & $74.0$\xsm{$\pm 0.5$} &  0.00 & $85.3$ \\
         & Random & $97.6$\xsm{$\pm 0.3$} &  0.00 & $70.0$\xsm{$\pm 0.3$} &  0.00 & $93.0$\xsm{$\pm 0.6$} &  0.00 & $89.7$\xsm{$\pm 0.4$} &  0.00 & $73.9$\xsm{$\pm 0.7$} &  0.00 & $84.8$ \\
        \midrule
        -  & Local & $99.1$\xsm{$\pm 0.1$} &  0.00 & $88.8$\xsm{$\pm 0.2$} &  0.00 & $98.9$\xsm{$\pm 0.1$} &  0.00 & $96.4$\xsm{$\pm 0.2$} &  0.00 & $88.8$\xsm{$\pm 0.2$} &  0.00 & $94.4$ \\
        \bottomrule
    \end{tabular}
    \begin{tabular}{cc*{12}{c}a}
        \toprule
        \rowcolor{white}
        \multicolumn{2}{c}{Sampling} & \multicolumn{2}{c}{Clipart} & \multicolumn{2}{c}{Infograph} & \multicolumn{2}{c}{Painting} & \multicolumn{2}{c}{Quickdraw} & \multicolumn{2}{c}{Real} & \multicolumn{2}{c}{Sketch} & Average \\
        \rowcolor{white}
        Budget & Method & Acc (\%) $\uparrow$ & $\epsilon \downarrow$  & Acc (\%) $\uparrow$ & $\epsilon \downarrow$  & Acc (\%) $\uparrow$ & $\epsilon \downarrow$  & Acc (\%) $\uparrow$ & $\epsilon \downarrow$  & Acc (\%) $\uparrow$ & $\epsilon \downarrow$ & Acc (\%) $\uparrow$ & $\epsilon \downarrow$ & Acc (\%) $\uparrow$ \\
        \midrule
        \multirow{3}{*}{ 1000}  & Ours & $\mathbf{88.4}$\xsm{$\pm 1.5$} &  0.58 & $\mathbf{52.6}$\xsm{$\pm 0.9$} &  0.58 & $\mathbf{90.4}$\xsm{$\pm 1.7$} &  0.58 & $\mathbf{84.3}$\xsm{$\pm 1.6$} &  0.58 & $92.1$\xsm{$\pm 1.2$} &  0.58 & $\mathbf{87.2}$\xsm{$\pm 0.5$} &  0.58 & $\mathbf{82.5}$ \\
         & K-Center & $86.8$\xsm{$\pm 0.3$} &  0.00 & $50.5$\xsm{$\pm 0.9$} &  0.00 & $89.1$\xsm{$\pm 1.4$} &  0.00 & $27.2$\xsm{$\pm 1.8$} &  0.00 & $\mathbf{92.5}$\xsm{$\pm 0.1$} &  0.00 & $85.6$\xsm{$\pm 1.4$} &  0.00 & $72.0$ \\
         & Random & $86.9$\xsm{$\pm 0.8$} &  0.00 & $47.4$\xsm{$\pm 2.7$} &  0.00 & $88.6$\xsm{$\pm 0.1$} &  0.00 & $77.9$\xsm{$\pm 2.4$} &  0.00 & $91.4$\xsm{$\pm 0.3$} &  0.00 & $86.9$\xsm{$\pm 0.5$} &  0.00 & $79.9$ \\
        \midrule
        \multirow{3}{*}{ 3000}  & Ours & $93.2$\xsm{$\pm 0.4$} &  0.58 & $\mathbf{58.1}$\xsm{$\pm 0.6$} &  0.58 & $92.5$\xsm{$\pm 1.1$} &  0.58 & $\mathbf{89.2}$\xsm{$\pm 0.9$} &  0.58 & $\mathbf{94.4}$\xsm{$\pm 0.2$} &  0.58 & $92.8$\xsm{$\pm 0.2$} &  0.58 & $\mathbf{86.7}$ \\
         & K-Center & $93.5$\xsm{$\pm 1.1$} &  0.00 & $56.3$\xsm{$\pm 0.3$} &  0.00 & $\mathbf{92.9}$\xsm{$\pm 0.3$} &  0.00 & $60.5$\xsm{$\pm 8.7$} &  0.00 & $94.1$\xsm{$\pm 0.2$} &  0.00 & $92.1$\xsm{$\pm 0.7$} &  0.00 & $81.6$ \\
         & Random & $92.5$\xsm{$\pm 0.6$} &  0.00 & $53.6$\xsm{$\pm 1.4$} &  0.00 & $91.7$\xsm{$\pm 0.8$} &  0.00 & $86.1$\xsm{$\pm 0.4$} &  0.00 & $93.5$\xsm{$\pm 0.3$} &  0.00 & $93.0$\xsm{$\pm 0.2$} &  0.00 & $85.1$ \\
        \midrule
        -  & Local & $87.0$\xsm{$\pm 0.3$} &  0.00 & $51.7$\xsm{$\pm 0.7$} &  0.00 & $85.9$\xsm{$\pm 0.4$} &  0.00 & $83.5$\xsm{$\pm 0.3$} &  0.00 & $93.5$\xsm{$\pm 0.1$} &  0.00 & $81.7$\xsm{$\pm 0.6$} &  0.00 & $80.6$ \\
        \bottomrule
    \end{tabular}
    \vspace{-0.25in}
\end{table}

\textbf{Case 1: Selective manual labeling}.
We assume that the cloud will label the filtered in-domain samples by using a third-party label service, e.g., Amazon Mechanical Turk~\cite{paolacci2010running}, or by asking the end clients for manual labeling.
As the selected samples are non-private, they can be freely shared with a third parity.
But the high cost of manual labeling or service is the pain point, which should be carefully constrained within a finite \emph{budget} of demanded labels.

Given a specified sampling budget, we compare the test accuracy of semi-supervised learning (FixMatch) on sampled data in \cref{tbl:alabel}.
Since the ECOS tends to select in-distribution samples, it eases the transfer of cloud-trained model to the client data.
On the Digits dataset, we find that our method attains more accuracy gains as budget increases, demonstrating that more effective labels are selected.
On the DomainNet dataset, our method outperforms baselines on 5 out of 6 domains and is stable in most domains given a small budget (1000) and is superior on average.
Given higher budgets, the accuracy of all methods are improved, and our method is outstanding on Infograph and Quickdraw domains and is comparable to the baselines on other domains.

\begin{table}[ht]
    \caption{Test accuracy by adaptive model compression on DomainNet.}
    \label{tbl:adaptation}
    \centering
    \scriptsize
    \setlength\tabcolsep{2 pt}
    \begin{tabular}{cc*{12}{c}a}
        \toprule
        \rowcolor{white}
        \multicolumn{2}{c}{Sampling} & \multicolumn{2}{c}{Clipart} & \multicolumn{2}{c}{Infograph} & \multicolumn{2}{c}{Painting} & \multicolumn{2}{c}{Quickdraw} & \multicolumn{2}{c}{Real} & \multicolumn{2}{c}{Sketch} & Average \\
        \rowcolor{white}
        Budget & Method & Acc (\%) $\uparrow$ & $\epsilon \downarrow$  & Acc (\%) $\uparrow$ & $\epsilon \downarrow$  & Acc (\%) $\uparrow$ & $\epsilon \downarrow$  & Acc (\%) $\uparrow$ & $\epsilon \downarrow$  & Acc (\%) $\uparrow$ & $\epsilon \downarrow$ & Acc (\%) $\uparrow$ & $\epsilon \downarrow$ & Acc (\%) $\uparrow$  \\
        \midrule %
        \multirow{3}{*}{ 1000}  & Ours & $82.9$\xsm{$\pm 0.7$} &  0.58 & $\mathbf{48.5}$\xsm{$\pm 2.7$} &  0.58 & $\mathbf{85.4}$\xsm{$\pm 1.0$} &  0.58 & $\mathbf{81.3}$\xsm{$\pm 2.5$} &  0.58 & $91.4$\xsm{$\pm 0.6$} &  0.58 & $\mathbf{82.4}$\xsm{$\pm 1.1$} &  0.58 & $\mathbf{78.6}$ \\
         & K-Center & $81.2$\xsm{$\pm 0.9$} &  0.00 & $44.6$\xsm{$\pm 0.9$} &  0.00 & $84.5$\xsm{$\pm 2.2$} &  0.00 & $41.7$\xsm{$\pm 1.9$} &  0.00 & $\mathbf{92.4}$\xsm{$\pm 0.5$} &  0.00 & $80.1$\xsm{$\pm 2.1$} &  0.00 & $70.8$ \\
         & Random & $\mathbf{83.8}$\xsm{$\pm 0.7$} &  0.00 & $44.4$\xsm{$\pm 2.1$} &  0.00 & $83.8$\xsm{$\pm 1.6$} &  0.00 & $76.3$\xsm{$\pm 2.2$} &  0.00 & $90.1$\xsm{$\pm 0.6$} &  0.00 & $80.1$\xsm{$\pm 0.7$} &  0.00 & $76.4$ \\
        \midrule %
        \multirow{3}{*}{ 3000}  & Ours & $\mathbf{90.6}$\xsm{$\pm 0.6$} &  0.58 & $\mathbf{51.4}$\xsm{$\pm 1.9$} &  0.58 & $\mathbf{89.6}$\xsm{$\pm 1.2$} &  0.58 & $\mathbf{87.6}$\xsm{$\pm 0.3$} &  0.58 & $93.6$\xsm{$\pm 0.7$} &  0.58 & $\mathbf{88.4}$\xsm{$\pm 1.5$} &  0.58 & $\mathbf{83.5}$ \\
         & K-Center & $88.9$\xsm{$\pm 2.3$} &  0.00 & $51.2$\xsm{$\pm 0.4$} &  0.00 & $89.4$\xsm{$\pm 0.9$} &  0.00 & $57.6$\xsm{$\pm 4.4$} &  0.00 & $\mathbf{94.5}$\xsm{$\pm 0.5$} &  0.00 & $86.9$\xsm{$\pm 0.6$} &  0.00 & $78.1$ \\
         & Random & $88.4$\xsm{$\pm 0.8$} &  0.00 & $47.6$\xsm{$\pm 1.9$} &  0.00 & $89.4$\xsm{$\pm 1.0$} &  0.00 & $84.7$\xsm{$\pm 0.2$} &  0.00 & $93.0$\xsm{$\pm 1.0$} &  0.00 & $86.3$\xsm{$\pm 1.2$} &  0.00 & $81.6$ \\
        \bottomrule
    \end{tabular}
\end{table}

\textbf{Case 2: Adaptive model compression}.
Due to the large volume of the open-source dataset, a larger model is desired for better capturing the various features, which however is so inefficient to fit into the resource-constrained client devices or specialize for the data distribution of the client.
Confronting this challenge, model compression~\cite{bucilua2006model} is a conventional idea to forge a memory-efficient model by transferring knowledge from large models to small ones.
Specifically, we first pre-train a large \emph{teacher} model $f_{t}$ on all cloud data by the supervised learning, assuming labels are available in advance.
Still, we use an ImageNet-pre-trained model to initialize the feature extractor $\phi(\cdot)$ of a \emph{student} model $f_s$.
Then the client will use the downloaded feature extractor $\phi$ to filter samples.
Here, we utilize knowledge distillation~\cite{hinton2015distilling} to finetune the $f_s$ %
with an additional classifier head upon the $\phi$.
On the selected samples, we train a linear classifier head for 30 epochs under the supervision of true labels and the teacher model $f_t$, and then fine-tune the full network $f_s$ for 500 epochs.
The major challenge comes from distributional biases between the multi-source open-source data and the client data.
Leveraging ECOS, we may sample data near the client distribution and reduce the bias in the %
follow-up compression process.

We simulate the case on the large-sized image dataset, DomainNet, which is demanding for large-scale networks, e.g., ResNet50, to effectively learn the complicated features.
Here, we compress ResNet50 into a smaller network, ResNet18, by using an adaptively selected subset of the cloud dataset.
We omit the experiment for digit images where a large model may not be necessary for such a small image size.
In \cref{tbl:adaptation}, %
we present the test accuracy of compressed ResNet18 using 3000 samples from DomainNet in finetuning.
With a small portion of privacy cost ($\epsilon<0.6$), our method improves the accuracy on Clipart and Real domains against the baselines.
Note that the model accuracy here is lower than label outsourcing in \cref{tbl:alabel}, and reason is that the supervisions from the larger models are just an approximation of the full dataset.
Without using the full dataset for compression, the training can be completed fast and responsively on the demand of a client.

\begin{table}[ht]
    \caption{Test accuracy of client labeling on two datasets: Digits (top) and DomainNet (bottom). 
    }
    \label{tbl:test_acc_closed_set}
    \centering
    \setlength\tabcolsep{2 pt}
    \scriptsize
    \begin{tabular}{cc*{10}{c}a}
        \toprule
        \rowcolor{white}
        \multicolumn{2}{c}{Sampling} & \multicolumn{2}{c}{MNIST} & \multicolumn{2}{c}{SVHN} & \multicolumn{2}{c}{USPS} & \multicolumn{2}{c}{SynthDigits} & \multicolumn{2}{c}{MNIST-M} & Average \\
        \rowcolor{white}
        Budget & Method & Acc (\%) $\uparrow$ & $\epsilon \downarrow$  & Acc (\%) $\uparrow$ & $\epsilon \downarrow$  & Acc (\%) $\uparrow$ & $\epsilon \downarrow$  & Acc (\%) $\uparrow$ & $\epsilon \downarrow$  & Acc (\%) $\uparrow$ & $\epsilon \downarrow$ & Acc (\%) $\uparrow$ \\
        \midrule
        \multirow{3}{*}{ 5000}  & Ours & $\mathbf{84.2}$\xsm{$\pm 2.3$} &  5.35 & $47.9$\xsm{$\pm 3.1$} &  5.32 & $\mathbf{86.1}$\xsm{$\pm 1.0$} &  5.35 & $68.6$\xsm{$\pm 1.6$} &  5.35 & $\mathbf{58.4}$\xsm{$\pm 1.9$} &  5.35 & $\mathbf{69.0}$ \\
         & K-Center & $81.9$\xsm{$\pm 3.4$} &  5.34 & $48.4$\xsm{$\pm 1.2$} &  5.33 & $82.1$\xsm{$\pm 1.2$} &  5.34 & $69.4$\xsm{$\pm 1.9$} &  5.34 & $55.4$\xsm{$\pm 2.0$} &  5.34 & $67.4$ \\
         & Random & $81.8$\xsm{$\pm 4.1$} &  5.34 & $45.3$\xsm{$\pm 3.0$} &  5.29 & $81.2$\xsm{$\pm 2.3$} &  5.34 & $65.9$\xsm{$\pm 2.7$} &  5.34 & $55.5$\xsm{$\pm 2.6$} &  5.34 & $65.9$ \\
        \bottomrule
    \end{tabular}
    \begin{tabular}{cc*{12}{c}a}
        \toprule
        \rowcolor{white}
          \multicolumn{2}{c}{Sampling} & \multicolumn{2}{c}{Clipart} & \multicolumn{2}{c}{Infograph} & \multicolumn{2}{c}{Painting} & \multicolumn{2}{c}{Quickdraw} & \multicolumn{2}{c}{Real} & \multicolumn{2}{c}{Sketch} & Average \\
         \rowcolor{white}
          Budget & Method & Acc (\%) $\uparrow$ & $\epsilon \downarrow$  & Acc (\%) $\uparrow$ & $\epsilon \downarrow$  & Acc (\%) $\uparrow$ & $\epsilon \downarrow$  & Acc (\%) $\uparrow$ & $\epsilon \downarrow$  & Acc (\%) $\uparrow$ & $\epsilon \downarrow$ & Acc (\%) $\uparrow$ & $\epsilon \downarrow$ & Acc (\%) $\uparrow$  \\
        \midrule
        \multirow{3}{*}{ 3000}  & Ours & $33.2$\xsm{$\pm 5.9$} &  4.46 & $\mathbf{23.8}$\xsm{$\pm 2.2$} &  3.50 & $\mathbf{47.4}$\xsm{$\pm 3.3$} &  4.51 & $\mathbf{39.8}$\xsm{$\pm 7.7$} &  4.87 & $62.9$\xsm{$\pm 1.9$} &  4.92 & $\mathbf{51.7}$\xsm{$\pm 1.2$} &  4.94 & $\mathbf{43.2}$ \\
         & K-Center & $\mathbf{39.3}$\xsm{$\pm 3.6$} &  4.57 & $18.2$\xsm{$\pm 2.7$} &  3.61 & $46.6$\xsm{$\pm 2.5$} &  4.52 & $36.1$\xsm{$\pm 2.9$} &  4.94 & $\mathbf{63.8}$\xsm{$\pm 1.9$} &  4.96 & $47.7$\xsm{$\pm 2.8$} &  4.96 & $42.0$ \\
         & Random & $30.7$\xsm{$\pm 2.2$} &  4.41 & $21.6$\xsm{$\pm 4.9$} &  3.43 & $44.0$\xsm{$\pm 3.1$} &  4.53 & $39.0$\xsm{$\pm 5.6$} &  4.82 & $59.9$\xsm{$\pm 3.6$} &  4.75 & $47.2$\xsm{$\pm 3.4$} &  4.94 & $40.4$ \\
         \bottomrule
    \end{tabular}
\end{table}

\textbf{Case 3: Automated client labeling}.
When the client obtained labeled samples, for example, photos labeled by phone users, the cloud only needs to collect an unlabeled public dataset.
Therefore, we may automate the labeling process leveraging the client supervision knowledge to reduce cost or users' efforts on manual labeling.
To be specific, we let the client generate pseudo labels for the cloud data based on their neighbor relation, as previously studied by \cite{zhu2020privateknn} (private kNN).
However, the private kNN assumes that the client and the cloud follow the same distribution, which weakens its applicability confronting the heterogeneity and the large scale of open-source data.
Thus, we utilize ECOS as a pre-processing of open-source data before being labeled by private kNN.
Therefore, we have two rounds of communication for transferring client knowledge: proximal-data sampling by the ECOS and client labeling by the private kNN \cite{zhu2020privateknn}.
\rev{To compose the privacy costs from these two steps, we utilize the analytical moment accountant technique to get a tight privacy bound~\cite{wang2019subsampled}.}
Interested readers can refer to \cref{sec:private_knn} for a brief introduction to private kNN and our implementations.

In experiments, we use the state-of-the-art private pseudo-labeling method, private kNN~\cite{zhu2020privateknn}, to label the subsampled open-source data with the assistance from the labeled client dataset.
To reduce the sensitivity of private kNN w.r.t. the threshold, we instead enforce the number of the selected labels to be 600 and balance the sizes by selecting the top-60 samples with the highest confidence per class.
Then, we adopt the popular semi-supervised learning method, FixMatch~\cite{sohn2020fixmatch}, to train the classifier.
In~\cref{tbl:test_acc_closed_set}, we report the results when cloud features are distributionally biased from the client ones but they share the same class set.
For each domain choice of client data, we will use the other domains as the cloud dataset.
Distinct from prior studies, e.g., in~\cite{zhu2020privateknn} or~\cite{papernot2018scalable}, %
we assume 80-90\% of the cloud data are out of the client distribution and are heterogeneously aggregated from different domains, casting greater challenges in learning.
On the Digits, we eliminate all ID data from the cloud set to harden the task.
Both on Digits and DomainNet datasets, our method consistently outperforms the two sampling baselines under the similar privacy costs.
The variance of privacy costs is mainly resulted from the actual sampling sizes.
Though simply adopting K-Center outperforms the random sampling, it still presents larger gaps compared to our method in multiple domains. For instance, in Quickdraw domain, given 3000 sampling budget, the K-Center method performs poorly and is even worse than a random sampling.

\subsection{Qualitative Study}

\begin{table}[b]
    \caption{Ablation study of the proposed method on DomainNet. Test accuracy of the client labeling case is reported.}
    \label{tab:ablation}
    \centering
    \setlength\tabcolsep{2 pt}
    \tiny
    \begin{tabular}{cc*{6}{c}a}
    \toprule
    \rowcolor{white}
    Proximity & Diversity &  Clipart & Infograph & Painting & Quickdraw & Real & Sketch & Average \\
    \midrule
    \xmark & \xmark & $30.7$\xsm{$\pm 2.2$} & $21.6$\xsm{$\pm 4.9$} & $44.0$\xsm{$\pm 3.1$} & $39.0$\xsm{$\pm 5.6$} & $59.9$\xsm{$\pm 3.6$} & $47.2$\xsm{$\pm 3.4$} & $40.4$ \\
    \cmark & \xmark & $25.5$ \xsm{$\pm 5.7$} & $21.1$\xsm{$\pm 0.5$} & $41.4$\xsm{$\pm 5.3$} & $31.3$\xsm{$\pm 1.3$} & $60.3$\xsm{$\pm 2.3$} & $31.9$\xsm{$\pm 2.9$} & $35.2$ \\
    \xmark & \cmark &$\mathbf{39.3}$\xsm{$\pm 3.6$} & $18.2$\xsm{$\pm 2.7$} & $46.6$\xsm{$\pm 2.5$} & $36.1$\xsm{$\pm 2.9$} & $\mathbf{63.8}$\xsm{$\pm 1.9$} & $47.7$\xsm{$\pm 2.8$} & $42.0$ \\
    \cmark & \cmark & $33.2$\xsm{$\pm 5.9$} & $\mathbf{23.8}$\xsm{$\pm 2.2$} & $\mathbf{47.4}$\xsm{$\pm 3.3$} & $\mathbf{39.8}$\xsm{$\pm 7.7$} & $\mathbf{62.9}$\xsm{$\pm 1.9$} & $\mathbf{51.7}$\xsm{$\pm 1.2$} & $\mathbf{43.2}$ \\
    \bottomrule
    \end{tabular}
\end{table}

To better understand the proposed method, we conduct a series of qualitative studies on client labeling.
We use two DomainNet datasets as a representative benchmark in the studies.
\begin{wrapfigure}{r}{0.35\textwidth}
    \centering
    \includegraphics[width=0.32\textwidth]{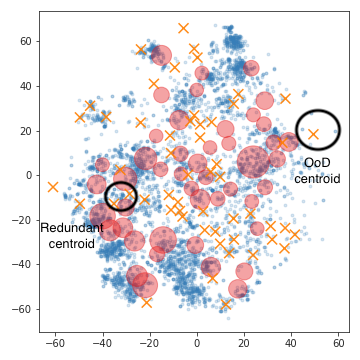}
    \caption{Demonstration of the centroids qualified by private CC. We use blue dots to represent the client data from Real domain of DomainNet. For the data of the cloud domains, larger circles represent centroids with higher CC and orange crosses are rejected OoD centroids.}
    \label{fig:DomainNet_covrage}
    \vspace{-0.1in}
\end{wrapfigure}
1)~\textbf{Ablation study}. In \cref{tab:ablation}, we conduct ablation studies to evaluate the effect of different objectives introduced by ECOS, following the client labeling benchmark with a 3000 sample budget.
Without proximity and diversity objectives, we let the baseline be the naive random sampling.
We first include the proximity objective, where we greedily select samples from top-scored clusters until the budget is fulfilled.
\rev{However, we find that the naive proximity objective results in a quite negative effect compared to the random baseline.
The failure can be attributed to the nature of clustering, which includes more similar samples and thereby lacks diversity.}
When diversity is encouraged and combined with the proximal votes, we find the performance is improved significantly in multiple domains and on average.
Now with diversity, the proximity objective can further improve the sampling in Painting, Quickdraw, and Sketch domains significantly.
2)~\textbf{Visualize cluster selection}.
In \cref{fig:DomainNet_covrage}, we demonstrate that the CC can effectively reject OoD centroids.
Also, it is interesting to observe that when multiple centroids are distributed closely, then they will compete with each other and reject the redundant ones consequently.
3)~\textbf{Efficiency}. In \cref{sec:method}, we studied the communication efficiency theoretically.
Empirically, ECOS only needs to upload 100 bytes of the CC scores in all experiments, while traditional outsourcing needs to upload the dataset, which is about 198MB for the lowest load in DomainNet (50\% of Sketch domain data for client).
Even counting the download load, ECOS only needs to download 45MB of the pre-trained ResNet18 feature extractor together with 51KB data of centroid features, which is much less than data uploading.
More detailed evaluation of the efficiency is placed in \cref{sec:app:eff} and we study how sample budgets and privacy budgets trade off the accuracy of ECOS models in \cref{sec:app:smp_eff}.

\section{Conclusion}
\vspace{-0.1in}

In this paper, we explore the possibility of outsourcing model training without access to the client data.
To reconcile the data shortage from the target domain, we propose to find a surrogate dataset from the source agnostic public dataset.
We find that the heterogeneity of the open-source data greatly compromises the performance of trained models.
To tackle this practical challenge, we propose a collaborative sampling solution, ECOS, that can efficiently and effectively filter open-source samples and thus benefits follow-up learning tasks.
We envision this work as a milestone for the private and efficient outsourcing from low-power and cost-effective end devices.

\rev{
We also recognize open questions of the proposed solution for future studies.
For example, the public dataset may require additional data processing, e.g., aligning and cropping for improved prediction accuracy.
In our empirical studies, we only consider the computer vision tasks, though no assumption was made on the data structures. 
We expect the principles to be adapted to other data types with minimal efforts.
More data types, including tabular and natural-language data, will be considered in the follow-up works.
Take the language data for example:
BERT is among the most popular pre-trained models for extracting features from sentences~\cite{devlin2018bert}, upon which ECOS can assistant to sample proxy data from massive public dataset, for instance, WikiText-103~\cite{merity2016pointer} based on 28,475 Wikipedia articles.
Alternative to the $L_2$-norm based proximity objectives, advanced semantic features could also enhance the sampling effectiveness of ECOS in language data~\cite{zad2021survey}. 
More discussions on the social impacts of this work are enclosed in \cref{sec:app:dicussion}.
}

\begin{ack}
This research was funded by Sony AI. 
This material is based in part upon work supported by the National Science Foundation under Grant IIS-2212174, IIS-1749940, Office of Naval Research N00014-20-1-2382, and National Institute on Aging (NIA) RF1AG072449.
\end{ack}

\bibliographystyle{abbrv}
	\bibliography{auto_gen}  %

\clearpage

\appendix

\section{More Method Details}

In this section, we elaborate on additional technical and theoretical details of our paper.

\subsection{Automated client labeling via private kNN}
\label{sec:private_knn}

Here, we briefly introduce the main idea of client labeling by private kNN.
Given $C$ classes, labeling is done by nearest neighbor voting:
\begin{align*}
    f(x) = \argmax \nolimits_{c \in [C]} A_{\epsilon}(v_c), ~ v_c = |\{ (x', y') \in \NN_K(x) | y'=c \}|,
\end{align*}
where $\NN_K(x)$ is a set including the $K$-nearest neighbors of $x$ in the client dataset.
$A_{\epsilon}$ is a privatization mechanism complying with the notion of $\epsilon$-Differential Privacy (DP) \cite{dwork2006calibrating}.
In brief, privatization is done by adding Gaussian noise to a value with finite sensitivity.
To filter out potential wrong labels, the client only returns high-confident samples by screening.
Let the confidence of a pseudo label be $s(x) = \max_{c\in [C]}A_{\epsilon}(v_c)$ which is also privatized by the Gaussian noise mechanism.
We find that the original version of screening may suffer from a large imbalance of pseudo labels.
Per class, we screen the pseudo labels by selecting the top-$60$ confident samples given $600$ labeling budget.

\subsection{Improving client labeling}

Because of the noise mechanism for privacy protection, the client labeling may be quite random if the selected samples are hard to discriminate.
Thus, we propose to improve the discrimination of selected samples in advance, when the ECOS selects samples for labeling.
First, we estimate the discrimination by the confidence in the ECOS.
The ECOS confidence is defined by the vote count of the highest-voted class subtracting the one of the second highest one, denoted as $v^{\text{conf}}_r$.
To merit the balancedness of samples, we filter the clusters to keep the top-$70\%$ samples with the highest confidence \emph{per class}.
When decompressing the clusters on the cloud, we incorporate the confidence into the sampling score by $v'_r = \psi\left[ (v^{\text{conf}}_r  + v_r) /2 \right]$ where $v_r$ is the original score.

\subsection{Privacy Accountant for ECOS}
\label{sec:app:privacy_proof}

To understand the privacy cost of ECOS, we review the techniques that are essential to establish the privacy bound.
\begin{definition}[Differential Privacy~\cite{dwork2006calibrating}]
Suppose $\epsilon$ and $\delta$ are two positive constants.
A randomized algorithm $\cM: \cX \rightarrow \cY$ is $(\epsilon, \delta)$-DP if for every pair of neighboring datasets $X, X' \in \cX$, and every possible measurable output set $Y\subset \cY$ the following inequality holds:
\begin{align*}
    \text{Pr}[\cM(X) \in Y] \le e^\epsilon \text{Pr}[\cM(X') \in Y] + \delta,
\end{align*}
where $\text{Pr}[\cdot]$ denotes the probability of a given event.
\end{definition}
DP provides a way to quantify the privacy risk (termed as the difference between two probability given a pair of similar but different inputs) in the probability of $\delta$.
Though DP is a simple notion for risks, the estimation of a tight privacy bound is still challenging.
For this reason, RDP is proposed an alternative tool.
\begin{definition}(R\'enyi Differential Privacy (RDP)~\cite{mironov2017renyi})
A randomized algorithm $\cM: \cX \rightarrow \cY$ is $(\alpha, \epsilon)$-RDP with order $\alpha > 1$ if for all neighboring datasets $X, X'$ the following holds
\begin{align*}
    D_{\alpha}(\cM(X) || \cM(X')) \le \epsilon,
\end{align*}
where $ D_{\alpha}(\cdot || \cdot)$ is the R\'enyi divergence between two distributions.
\end{definition}
The RDP and DP can be connected by \cref{lm:RDP_to_DP}. 
\begin{lemma} \label{lm:RDP_to_DP}
If a mechanism $\cM$ satisfies $(\alpha, \epsilon)$-RDP, then it also satisfies $(\epsilon + \frac{\log 1/\delta}{\alpha - 1}, \delta)$-DP.
\end{lemma}

To reveal the potential privacy risks, we give a theoretical bound on the privacy cost based on DP in \cref{lm:eocs_privacy}.
The proof of \cref{lm:eocs_privacy} is similar to Theorem 8 in \cite{zhu2020privateknn} without confidence screening.
\begin{lemma}
\label{lm:eocs_privacy}
Suppose the subsampling rate $\gamma$ and noise magnitude $\sigma$ of the ECOS are positive values such that $\gamma \le \min\left\{0.1, \sigma\sqrt{\log(1/\delta)/6} \right\}$ and $\sigma \ge 2\sqrt{5}$.
The total privacy bound of the ECOS scoring $m=|\hat D^q|$ query samples with $n=|\hat D^p|$ private client samples is $(\epsilon, \delta)$-DP with $\delta > 0$ and 
\begin{align}
    \epsilon = \cO(\frac{\gamma}{\sigma} \sqrt{ \log(1/\delta)}),
\end{align}
if $v_r$ in \cref{alg:ours} is estimated by using $\lceil \gamma n \rceil$ samples randomly subsampled from $\hat D^p$ with replacement and is processed by $\tilde v_r = v_r + \cN(0, \sigma^2 I)$.
\end{lemma}
\begin{proof}
When one sample is absent from the private client dataset, the scores $[v_1, \dots, v_R]$ will differ by $2$ if without subsampling.
By Lemma 11 of \cite{zhu2020privateknn}, the subsampled Gaussian mechanism accounted by the RDP is
\begin{align*}
    \epsilon(\alpha) \le \frac{24 \gamma^2 \alpha}{\sigma^2}
\end{align*}
for all $0 < \alpha \le \frac{\sigma^2 \log (1/\gamma)}{2}$, $\gamma \le 0.1$ and $\sigma \ge 2\sqrt{5}$.
By \cref{lm:RDP_to_DP}, we can convert the RDP inequality to the standard $(\epsilon, \delta)$-DP as
\begin{align*}
    \epsilon = \frac{24 \gamma^2 \alpha}{\sigma^2} + \frac{\log(1/\delta)}{\alpha - 1}.
\end{align*}
Let $\alpha$ be $1 + \frac{\sqrt{\log(1/\delta)}}{\sqrt{\frac{24 \gamma^2}{\sigma^2}}}$.
Thus,
\begin{align*}
    \epsilon = \frac{24 \gamma^2}{\sigma^2} + 4 \frac{\gamma}{\sigma} \sqrt{6 \log(1/\delta)} \le 8 \frac{\gamma}{\sigma} \sqrt{6 \log(1/\delta)},
\end{align*}
where the last inequality is derived by the given range of $\gamma$. This thus completes the proof. 
\end{proof}
The above bound implies that the privacy cost for our method is invariant w.r.t. the scale of the query dataset $\hat D^q$, and only depends on the DP parameters.
Note that the \cref{lm:eocs_privacy} is an asymptotic bound which requires some strict conditions on $\gamma$ or other variables.
In practice, we leverage the tool of analytic privacy accountant through a numerical method~\cite{wang2019subsampled}, with which we can relax the strict conditions.

\rev{
\subsection{Theoretical Analysis}
\label{sec:app:thm}

Though empirical results show that more accurate models can be trained on ECOS-sampled datasets, it remains unclear how the cloud dataset and the training process affect the model performance on the client.
In this section, we provide the proof of \cref{thm:error_ecos} using the domain generalization bound as stated in \cref{thm:da}.

\begin{theorem}[Domain-adaptation learning bound from \cite{blitzer2007learning}]
\label{thm:da}
    Suppose two domains with distributions $\cD_s$ and $\cD_t$.
    Let $\cH$ be a hypothesis space of $VC$-dimension $d$ and $D_s$ be the dataset induced by samples of size $N$ drawn from $\cD_s$, respectively.
    Then with probability at least $1 - p$ over the choice of samples, for each $f\in \cH$,
    \begin{align}
        L(f, \cD_t) &\le L(f, D_s) + {1\over 2} d_{\cH \Delta \cH} (\cD_{s}, \cD_t) + 4 \sqrt{2d \log (2 N) + \log(4/p) \over N} + \xi,
        \label{eq:da}
    \end{align}
    where $\xi$ is the optimal loss, i.e., $\min_f L(f, \cD_s) + L(f, \cD_t)$, and $d_{\cH \Delta \cH}(\cD_{s}, \cD_t)$ denotes the divergence between domain $s$ and $t$.
\end{theorem}

\begin{proof}[Proof of \cref{thm:error_ecos}]
Our proof is mainly based on \cref{thm:da}.
By definition, $S$ contains data sampled from the mixture of distributions, $\sum_{r=1}^R \cD_r^q$. 
Apply \cref{thm:da} to attain
\begin{align}
    L(f_{\theta}, \cD^p) &\le L(f_{\theta}, S) %
        + \frac{1}{2} d_{\cH \Delta \cH} (\sum_{r=1}^R \cD_r^q, \cD^p) + 4 \sqrt{2d \log (2 |S|) + \log(4/p) \over |S|} + \xi. \label{eq:da_S}
\end{align}
And we have
\begin{align}
    d_{\cH \Delta \cH} \left(\sum\nolimits_{r=1}^R \cD_r^q, \cD_t \right) \le \sum\nolimits_{r=1}^R \alpha_r d_{\cH \Delta \cH} (\cS_{r}, \cD_t), \label{eq:div_sum}
\end{align}
which was proved in~\cite{peng2019federated} (proof of Theorem 2).
Plugging \cref{eq:div_sum} into \cref{eq:da_S}, we can get the conclusion.
\end{proof}

Now, we consider $f_\theta$ to be trained by the widely-adopted gradient descent method.
We first present \cref{thm:conv} which characterizes the empirical error bound after $T$ iterations.
Substitute \cref{eq:conv} into \cref{eq:main_bd_0}.
Then we can obtain \cref{thm:error_ecos_convex}.

\begin{theorem}[Convergence bound, e.g., from \cite{boyd2004convex}] \label{thm:conv}
    Suppose the model $f$ is parameterized by $\theta$ initialized as $\theta_0$.
    Let the loss function $L(f_\theta, D)$ be $\mu$-strongly convex and $G$-smooth w.r.t. $\theta$, and assume that the global minima $\theta^*$ exists. Then after $T$ iterations, gradient descent with a fixed learning rate $\eta \le 1/G$ satisfies
    \begin{align}
        L(f_{\theta_T}, D) \le L^* + (1 - \mu/G)^T \ell_0(\theta_0, D) \label{eq:conv}
    \end{align}
    where $L^* = \min_\theta L(f_\theta, D)$ and $\ell_0(\theta_0, D)$ is the initial error gap, i.e., $| L(f_{\theta_0}, D) - L^* |$.
\end{theorem}

\begin{lemma}\label{thm:error_ecos_convex}
Suppose assumptions in \cref{thm:error_ecos,thm:conv} holds.
Let $L(\cdot, \cdot)$ be a loss function on a hypothesis and a dataset (for empirical error) or a distribution (for generalization error).
If $f$ is governed by the parameter $\theta$ trained on $S$ via $T$-iteration gradient descent and belongs to a hypothesis space $\cH$ of $VC$-dimension $d$,
then with probability at least $1 - p$ over the choice of samples, 
the inequality holds,
    \begin{align}
        L(f_{\theta_T}, \cD^p) &\le 
        (1 - \frac{\mu}{G})^T \ell_0(\theta_0, S) 
        + \frac{1}{2} \sum\nolimits_{r=1}^R \alpha_r d_{\cH \Delta \cH} (\cD^q_{r}, \cD^p) \notag \\
        &\quad + 4 \sqrt{2d \log (2 |S|) + \log(4/p) \over |S|} + \xi', \label{eq:main_bd_convex}
    \end{align}
    where $\xi' = \min_\theta \left\{ L(f_\theta, S) + L(f_\theta, D^p) \right\} + \min_\theta L(f_\theta, S)$, and $d_{\cH \Delta \cH}(\cD, \cD')$ denotes the distribution divergence.
\end{lemma}

}

\subsection{Social Impacts}
\label{sec:app:dicussion}

The conflict between the concerns on data privacy and demands for intensive computation resources for machine learning has composed the main challenge in training outsourcing.
In this work, we devote our efforts to outsourcing with uploading data by leveraging authorized or public datasets.
As the public datasets commonly available in many applications are collected from multiple data sources and thus tend to be non-identically distributed as the client data, it casts new challenges to use the public in place of the client dataset.
Our method addresses this problem with accountable privacy cost and low communication and computation complexity. Therefore, the proposed ECOS provides a promising solution to mitigate the aforementioned conflict between the privacy and computation desiderata.
Therefore, users from a broader spectrum can benefit from such a method to confidentially conduct cloud training.

\subsection{Connection to Federated Learning}

\rev{Both our method and federated learning (FL)~\cite{mcmahan2017communication} consider protecting data privacy via not sharing data with the cloud. 
The key difference between FL and our concerned problem (training outsourcing) is that FL requires clients to conduct training while ECOS outsources the training to the cloud server. 
Since ECOS does not require local training, it can be ad-hocly plugged into FL to obtain an auxiliary open-source dataset for enhancing the federated training. 
ECOS can be used either before federated training (when a pre-trained model is required) or during federated training (when the pre-trained model can be replaced by the on-training model).}

\section{Experimental Details and More Experiments}
\label{sec:app:exp}

Complementary to the main content, we provide the details of the experiment configurations to merit the reproducibility.
We also conduct more qualitative experiments to understand the efficiency and effectiveness of the proposed method.

\subsection{Experimental Details}
\label{sec:exp_details}

We organize the case-specified configurations into three cases and discuss the general setups first.

For Digits, we train the model for 150 epochs.
We adopt a convolutional neural network for Digits in \cref{tab:conv_model} and ResNet18 for DomainNet.
To solve the learning problems including FixMatch and distillation-based compression, we use stochastic gradient descent with the momentum of 0.9 and the weight decay of $5\times 10^{-4}$.
We use $s=5$ for the scale function $\psi_s$ on DomainNet and $s=1$ on Digits.
When not specified, we noise the ECOS query with the magnitude as $25$.
In Case 1 and 2, we reduce the noise magnitude to $10$ for DomainNet, since the two queries can bear more privacy costs to trade for higher accuracy.

\textbf{Case 1: Selective manual labeling}.
We make use of the off-the-shelf ResNet18 is pre-trained on the ImageNet, which is widely accessible online.
We adopt FixMatch for semi-supervised learning with the coefficient of $0.1$ on the pseudo-labeled loss, the moving average factor of $0.9$, and the batch size of 64 for DomainNet and 128 for Digits.
To avoid feature distorting, we warm up the fine-tuning by freezing all layers except the last linear layer with a learning rate of 0.01. 
After 30 epochs, we fine-tune the model end to end until 80 epochs to avoid overfitting biased data distributions.

\textbf{Case 3: Adaptive model compression}.
We first pre-train a ResNet50 using all labeled open-source data for 100 epochs with a cosine-annealed learning rate from $0.1$.
The same warm-up strategy as Case 1 is used here.
To extract the knowledge from ResNet50, we combine the knowledge-distillation (KD) loss $L_{KD}$ and cross-entropy loss $L_{CE}$ by $0.1 \times L_{KD}+0.9 \times L_{CE}$ and calculate the losses on the selected samples only.
The temperature in the KD loss is set to be $10$.

\textbf{Case 2: Automated client labeling}.
For the cloud training, we adopt the same configuration as the selective manual labeling.
For private kNN, we let the client release 600 labels with class-wise confidence thresholds described in the last section.
We noise the labeling in the magnitude of $25$ and the confidence in the magnitude of $75$.
For both datasets, we subsample $80\%$ client data per labeling query to reduce the privacy cost.

\begin{table}[ht]
    \caption{The structure of the conventional neural network for the Digits dataset.}
    \label{tab:conv_model}
    \centering
    \tiny
    \begin{tabular}{cc}
        \toprule
        Layer name & PyTorch pseudo code \\
        \midrule
        conv1 & Conv2d(1, 64, kernel\_size=(5, 5), stride=(1, 1)) \\
        bn1 & BatchNorm2d(64, eps=1e-05, momentum=0.1) \\
        conv1\_drop & Dropout2d(p=0.5, inplace=False) \\
        conv2 & Conv2d(64, 128, kernel\_size=(5, 5), stride=(1, 1)) \\
        bn2 & BatchNorm2d(128, eps=1e-05, momentum=0.1) \\
        conv2\_drop & Dropout2d(p=0.5, inplace=False) \\
        fc1 & Linear(in\_features=2048, out\_features=384, bias=True) \\
        fc2 & Linear(in\_features=384, out\_features=192, bias=True) \\
        fc3 & Linear(in\_features=192, out\_features=11, bias=True) \\
        \bottomrule
    \end{tabular}
\end{table}

We conduct our experiments on the Amazon Web Service platform with 4 Tesla T4 GPUs with 16GB memory and a 48-thread Intel CPU.
All the code is implemented with \texttt{PyTorch} 1.11.
To account for the privacy cost, we utilize the open-sourced \texttt{autodp} package following the private kNN.

\subsection{Effect of Parameters}
To better understand the proposed method, we study the effect of the important hyper-parameters.
To this end, we consider the selective labeling task with Digits, keeping $50\%$ of the SVHN dataset at the client end.
Both the ID+OoD and OoD cases are evaluated to reveal the method's effectiveness under circumstances with various hardness.
Also, we study how the score scale $s$ affects the ID ratios (denoted as the ID TPR) in the selected set and the number of effective samples.
We only examine the ID TPR corresponding to the proximity objective in \cref{eq:ecos_obj} if known ID samples are present on the cloud, namely in the \emph{ID+OoD} case.
In the middle panes of \cref{fig:Digits_alabel_vary_compress_k,fig:Digits_alabel_vary_vote_scale}, we show that our method can effectively improve the ID TPR against the inherent ratio of ID samples on the cloud.

\begin{figure}[ht]
    \centering
    \includegraphics[width=0.39\textwidth]{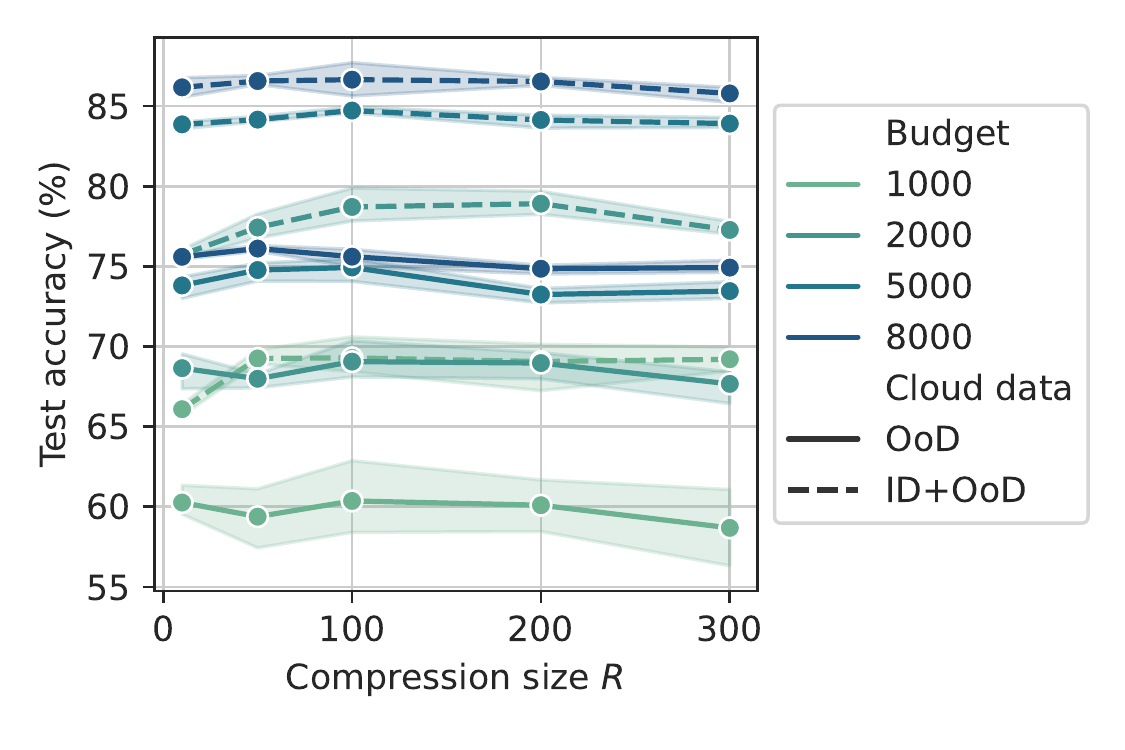} \hfill
    \includegraphics[width=0.29\textwidth]{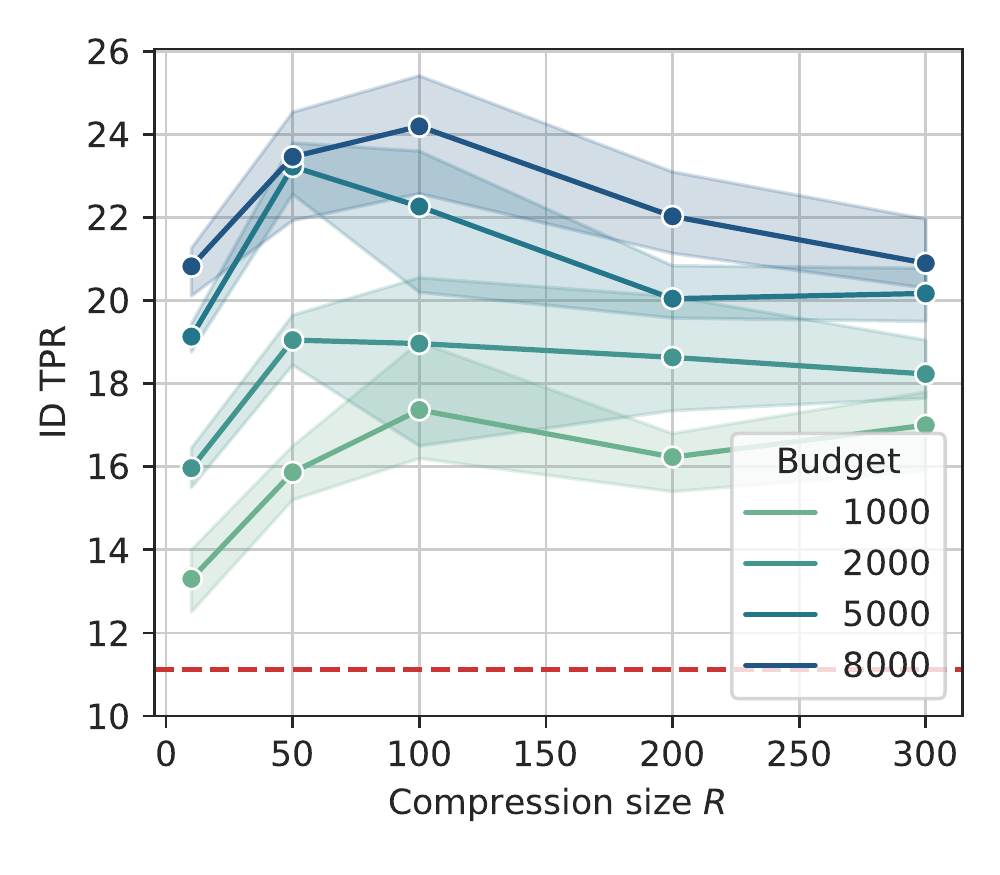} \hfill
    \includegraphics[width=0.29\textwidth]{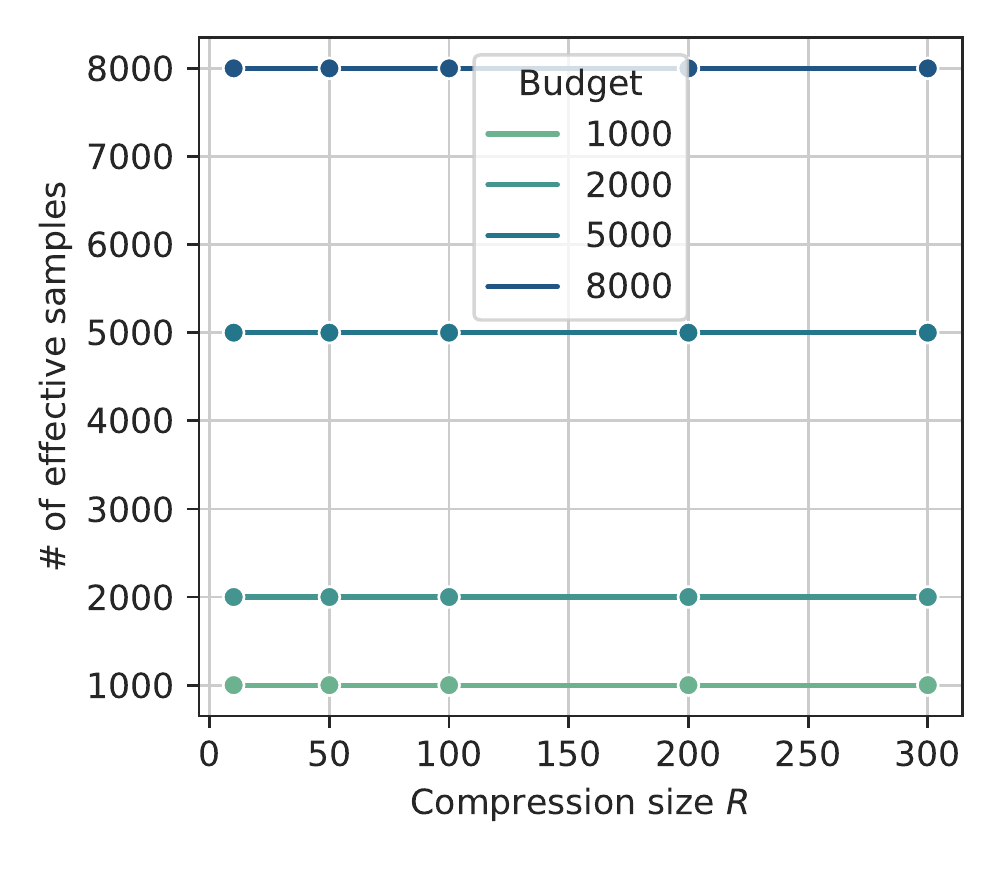}
    \caption{Vary the compression size $R$ and evaluate the test accuracy, ID ratios (\%) in selected samples and the number of effectively selected samples. The red horizontal line indicates the ratio of ID samples in the whole cloud dataset.}
    \label{fig:Digits_alabel_vary_compress_k}
\end{figure}

\textbf{Effect of the compression size $R$}.
In \cref{fig:Digits_alabel_vary_compress_k}, we evaluate $R$ in terms of the test accuracy.
When the budget is small (1k and 2k budgets in the ID+OoD case), it is essential for the cloud server to sense the client distribution with higher accuracy via more queries.
Therefore, a larger $R$ is desired, which can increase the portion of ID samples in the selected set, as shown in the middle pane of \cref{fig:Digits_alabel_vary_compress_k}.
Considering that a higher burden on communication comes with a larger $R$, the value of $100$ leads to a fair trade-off to the accuracy in which case the ID TPR reaches a peak.

Given a larger budget, e.g., 8000, increasing $R$ may lower the ID TPR.
We attribute the decline to the limited size of the client dataset and privacy constraints.
Given more clusters (i.e., $R$), the expected number of votes (proportional to the score) for each cluster will be reduced and is badly blurred by the DP noise.
Thus, the ID TPR will decrease simultaneously, regardless of which the test accuracy is not significantly affected.

For the OoD case which is relatively harder for sampling due to the lack of true ID data, the parameter sensitivity is weakened, though the compression size of $100$ is still a fair choice, for example, bringing in $1-2\%$ gains in the 5k, 8k cases comparing the worst cases.

\begin{figure}[ht]
    \centering
    \includegraphics[width=0.39\textwidth]{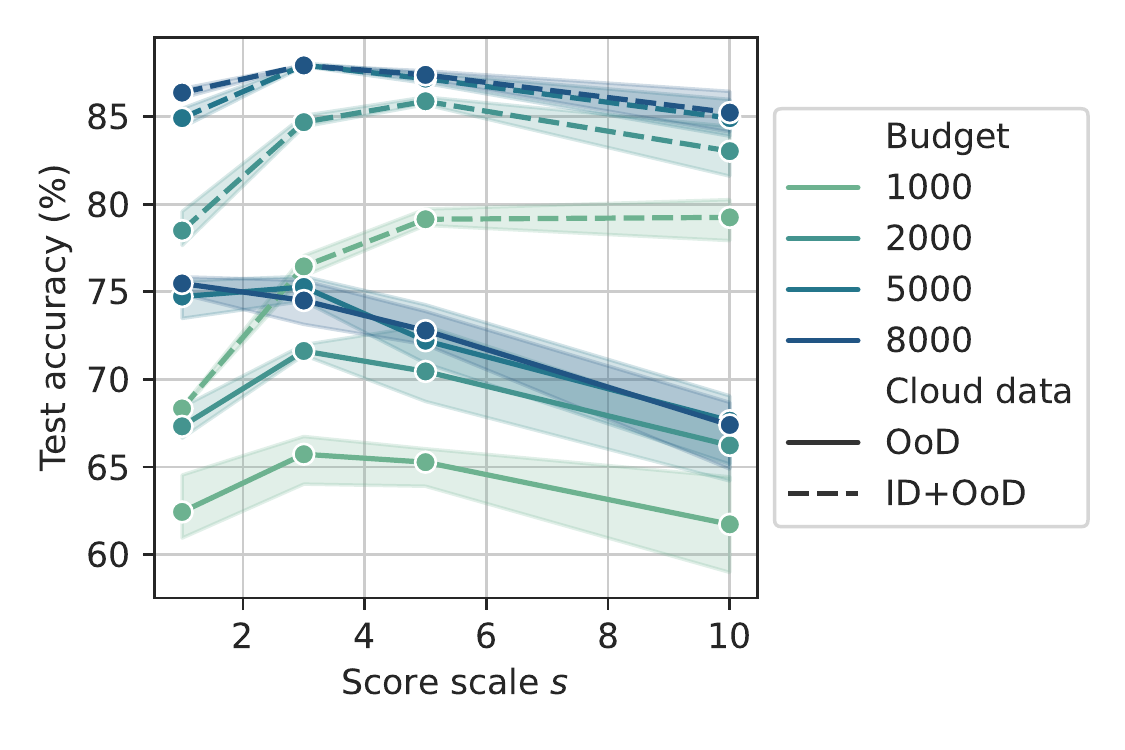} \hfill
    \includegraphics[width=0.29\textwidth]{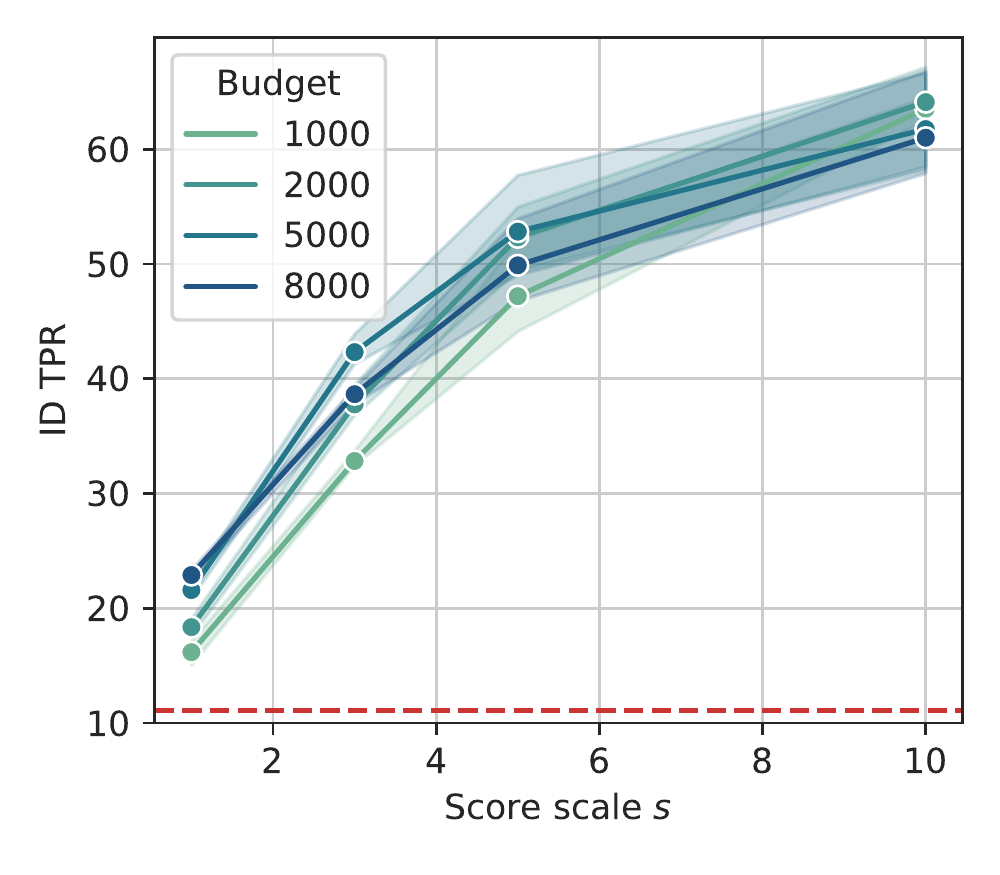} \hfill
    \includegraphics[width=0.29\textwidth]{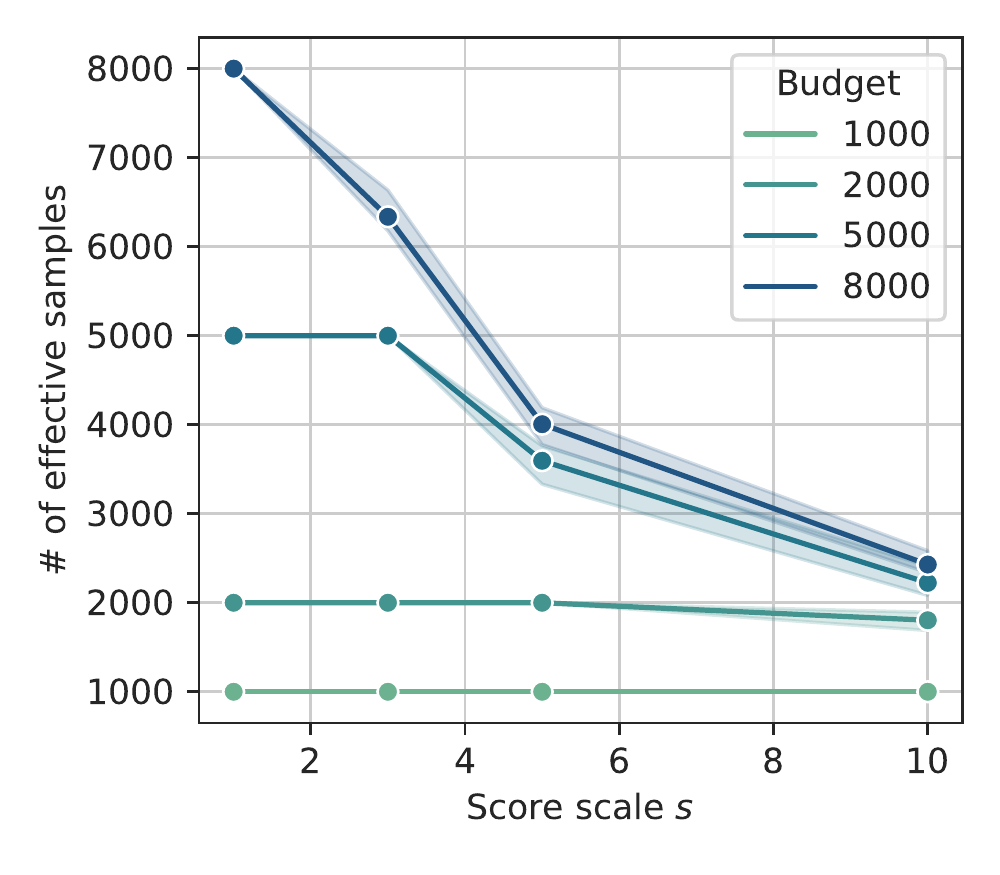}
    \caption{Vary the score scale $s$ in terms of test accuracy, ID ratios (\%) in selected samples and the number of effectively selected samples (which could be smaller than the budget). The red horizontal line indicates the ratio of ID samples in the whole cloud dataset.}
    \label{fig:Digits_alabel_vary_vote_scale}
\end{figure}

\textbf{Effect of the score scale $s$}.
The score scale $s$ decides the sensitivity of the sampling in the sense of proximity.
A larger $s$ means that the ECOS will prioritize the proximity more during sampling.
In \cref{fig:Digits_alabel_vary_vote_scale}, we present the ablation study of $s$.
A larger $s$ is preferred when the budget becomes limited because it increases the ID TPR effectively.
Though not significantly, an overly large $s$ has a significantly negative influence on the accuracy, especially for the OoD case.
The reason for the negative impact of $s$ on a large budget can be understood by probing the number of effective samples.
For budgets larger than 2000, the effectively selected samples are reduced with heavily scaled scores (e.g., $s\ge 3$) where the ECOS will concentrate its selection into very few clusters and eliminate the rest clusters strictly.

\subsection{Evaluation of Sample and Privacy Efficiency}
\label{sec:app:smp_eff}

\textbf{Effects of sample budgets.}
In \cref{fig:Digits_alabel_vary_budget}, we compare the sample efficiency in the selective labeling task with Digits, keeping $50\%$ of the SVHN dataset at the client end.
We obtain the upper-bound accuracy in the ideal case via random sampling when the cloud dataset distributes identically (ID) as the client dataset.
When OoD data are included in the open-source cloud dataset (ID+OoD), the training becomes more demanding for the labeled samples.
If none of the iid samples presents in the cloud set (OoD), the accuracy decreases quickly with the same labeled samples.
In comparison, informative sampling by K-Center slightly improves the accuracy by different budgets and the proposed ECOS significantly promotes the sample efficiency.
With ECOS, $8\times 10^3$ labeled samples in the ID+OoD case achieves comparable accuracy as the ideal case, while baselines remain large gaps.
Both in ID+OoD and OoD cases, our method yields competitive accuracy (at the $4\times 10^3$ budget) versus the best baseline results using only half of the labeled data (at the $1\times 10^4$ budget), dramatically cutting down the cost for manual labeling.

\begin{wrapfigure}{r}{0.45\textwidth}
    \vspace{-0.4in}
    \centering
    \includegraphics[width=0.43\textwidth]{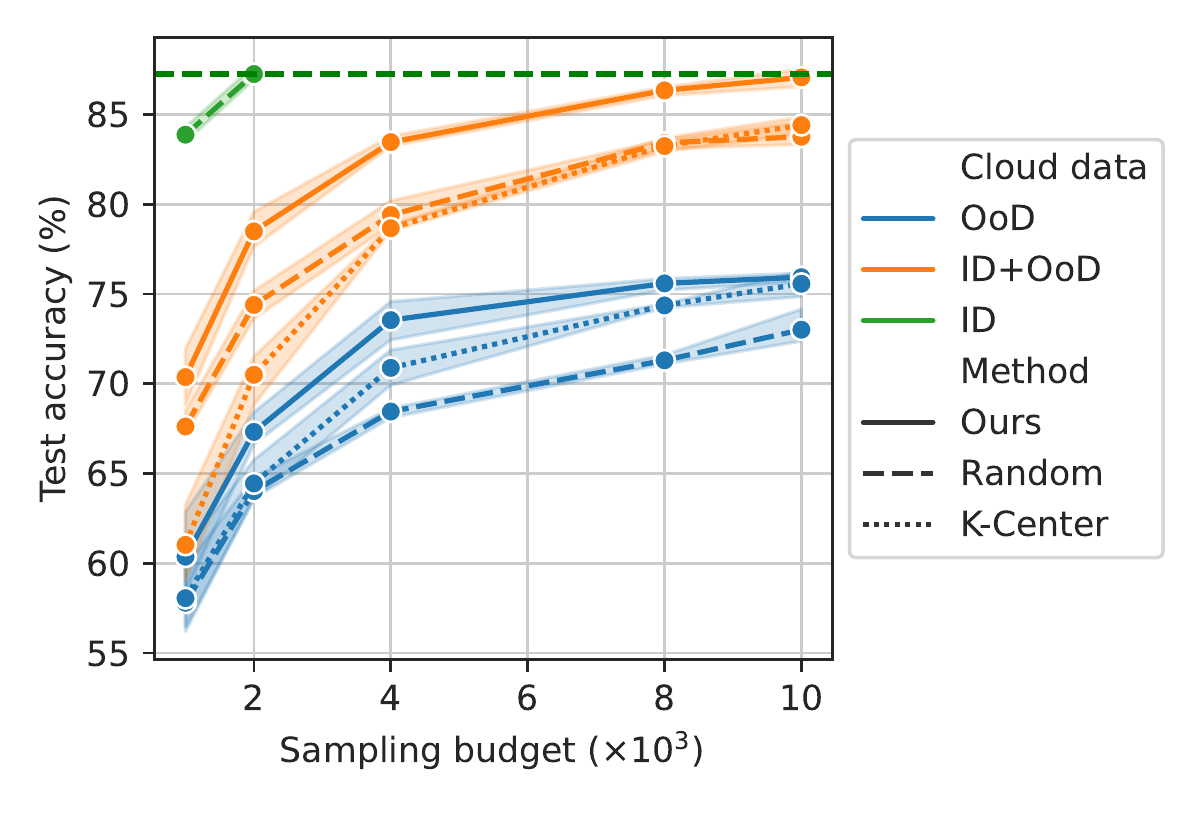}
    \vspace{-0.1in}
    \caption{Evaluation of the sample efficiency on selective labeling. The green horizontal line implies the ideal case when all ID cloud data are labeled.}
    \label{fig:Digits_alabel_vary_budget}
    \vspace{-0.2in}
\end{wrapfigure}

On observing the gains in sample efficiency, readers may also notice that our method induces additional costs at privacy, as compared to the baselines.
We point out that the cost is constant w.r.t. the sampling budget and is as neglectable as $(0.22, 10^{-5})$-DP.
It is worth noticing that the cost is independent of the hyper-parameters of the ECOS because the ECOS communication is a \emph{single} query for each private sample (so as for the private dataset), even if we increase the size of the query set (i.e., the compression size $R$).
In practice, the client can control the privacy risk (namely, the privacy cost) flexibly by adjusting the noise magnitude and the subsampling rate.

\begin{wrapfigure}{r}{0.45\textwidth}
    \vspace{-0.2in}
    \centering
    \includegraphics[width=0.43\textwidth]{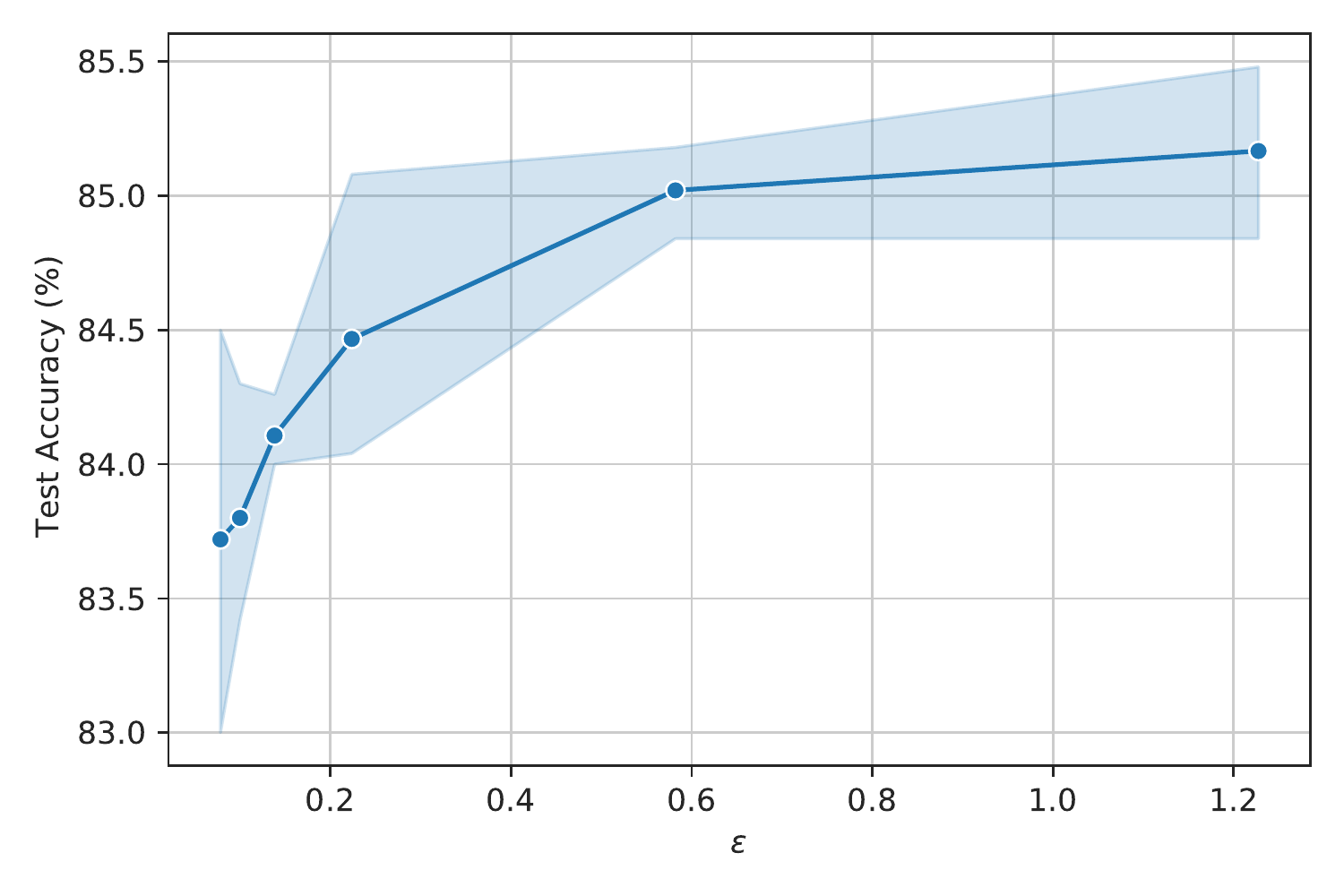}
    \vspace{-0.1in}
    \caption{Evaluation on how the privacy costs affect the performance.}
    \label{fig:acc_eps}
    \vspace{-0.3in}
\end{wrapfigure}

\textbf{Privacy-accuracy trade-off.}
Additionally, we evaluate how the privacy costs affect the 
In \cref{fig:acc_eps}, we study the relationship between privacy cost and performance by varying the noise scale $\sigma$ of the ECOS in label outsourcing. The experiment is conducted on SVHN client data for label outsourcing. Interestingly, the sampling effectiveness is not very sensitive to the noise scale. A very low privacy cost ($0.08$) can be achieved with noise as large as 70 and accuracy as high as $83.7\%$. The success could be attributed to the low dimension of queries (100 clusters) to the private dataset, resulting in the efficiency of privacy costs.

\subsection{Evaluation of Communication and Computation Efficiency}
\label{sec:app:eff}

When improving the sample efficiency, we also need to take care of the communication and computation overheads brought by the ECOS.
We examine the two kinds of efficiency by the same experiment configurations as in the last section.

\begin{figure}[ht]
    \centering
    \vspace{-0.2in}
    \begin{subfigure}[b]{0.45\textwidth}
        \includegraphics[width=\textwidth]{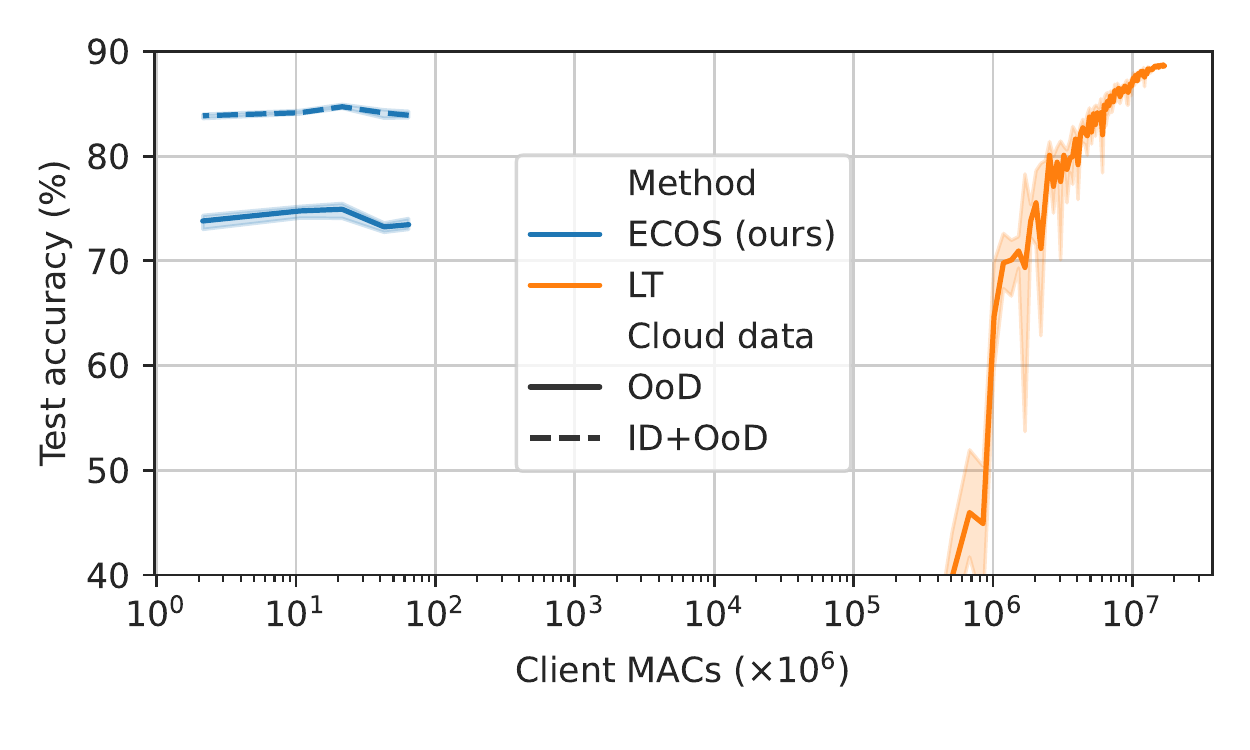}
        \vspace{-0.3in}
        \caption{Digits}
    \end{subfigure}
    \begin{subfigure}[b]{0.45\textwidth}
        \includegraphics[width=\textwidth]{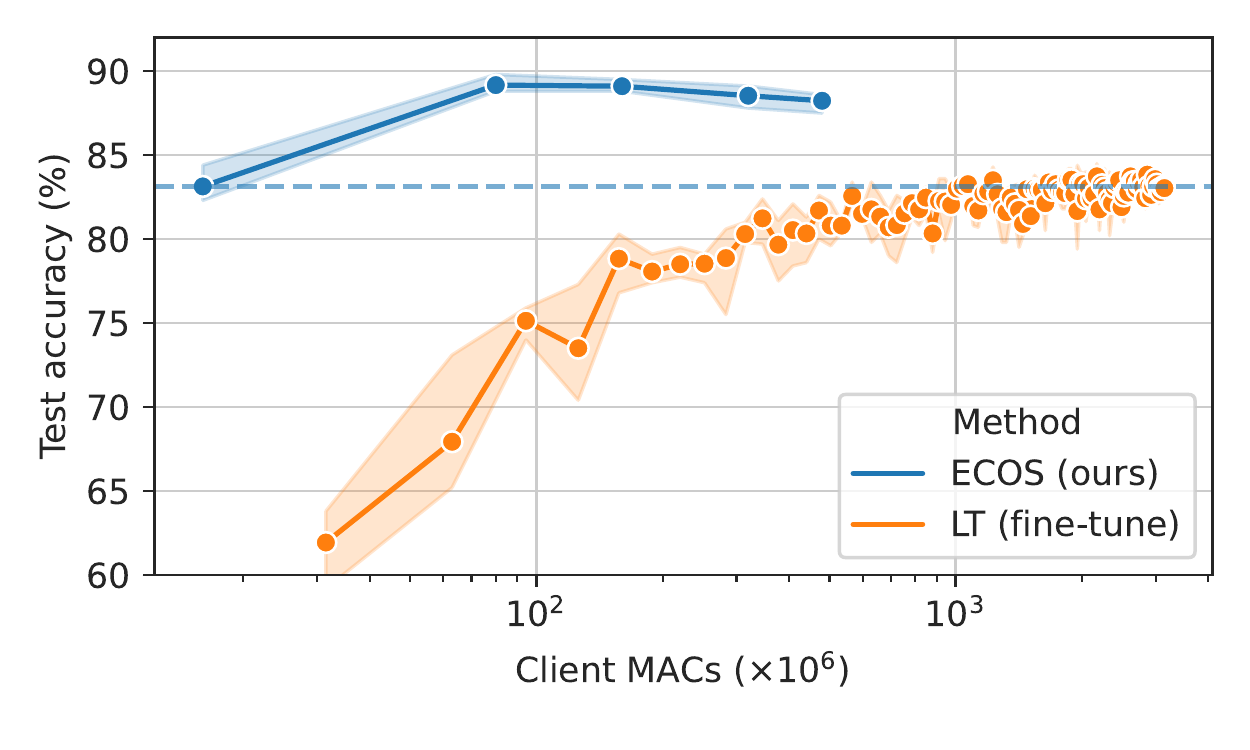} %
        \vspace{-0.3in}
        \caption{DomainNet}
    \end{subfigure}
    \caption{With the $5000$ budget, we evaluate the computation efficiency. The efficiency of locally-training (LT) on the DomainNet is enhanced by linear fine-tuning where only the linear classifier head is locally trained.
    }
    \label{fig:Digits_alabel_acc_v_comp_eff_vary_compress_k}
    \vspace{-0.1in}
\end{figure}

\begin{wrapfigure}{r}{0.4\textwidth}
    \vspace{-0.2in}
    \centering
    \includegraphics[width=0.3\textwidth]{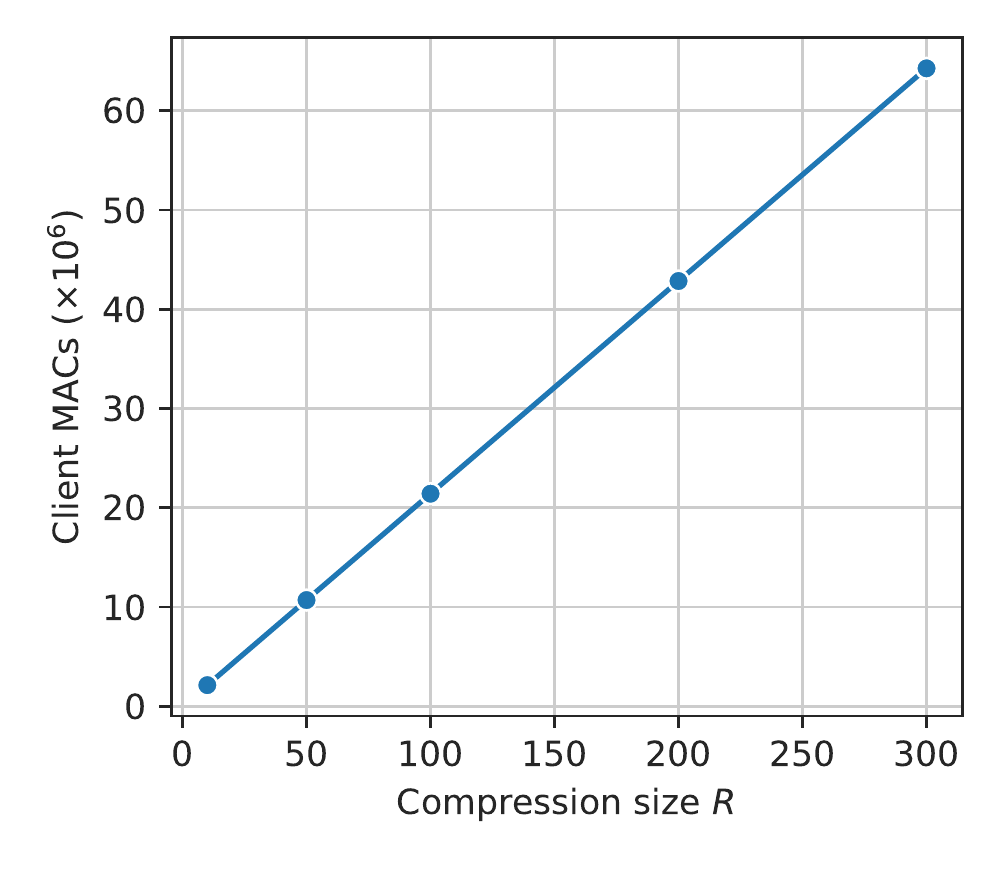}
    \caption{The linearly growing computation cost by increasing the compression size on the Digits dataset.}
    \label{fig:compress_k_comp_eff}
    \vspace{-0.2in}
\end{wrapfigure}

\textbf{Computation efficiency.}
In \cref{fig:Digits_alabel_acc_v_comp_eff_vary_compress_k}, we compare the computation efficiency of our method to the local training (LT).
We utilize the multiplication-and-addition counts (MACs) as the metric of computation (time) complexity, which is hardware-agnostic and therefore is preferred here.
For a fair comparison, we tune the learning rate in $\{0.1, 0.01, 0.001\}$ with the cosine annealing
during training and the number of epochs in $\{20, 50, 100\}$ of the LT to achieve a fair trade-off between the computation cost and test accuracy.
For ECOS, since the computation cost linearly increases by the compression size 
(as shown in \cref{fig:compress_k_comp_eff}), we vary the compression size to check the performance when increasing computation costs.
On Digits, we observe a large computation save by our method, even if the cost of our method will gradually increase by the size $R$ of the compressed query set.

Similar experiments are also run on the large-sized images using the DomainNet dataset (ID+OoD case), where the cost for extracting features is steeply increased by using a deep network (ResNet18).
Recently, the most popular strategy for cloud training is two-phase learning: pre-training a model on the cloud using ImageNet and fine-tuning the linear classifier head on the client.
Considering the large cost of feature extraction, we only let the client pre-extract features once only.
Thus, the local training is as efficient as training a \emph{linear} layer on extracted features.
In \cref{fig:Digits_alabel_acc_v_comp_eff_vary_compress_k}, our method outperforms the local training a lot using much fewer MACs for data matching.
Because all training is outsourced to the cloud, our method enables the end-to-end fine-tuning of the model resulting in better test accuracy.
Even if the LT trains longer with higher computation costs, the test accuracy of the ECOS with the least MACs is comparable to the best performance of LT at around $10^9$ MACs, where the ECOS only utilizes the 10\% of MACs by LT.

\begin{figure}[ht]
    \centering
    \includegraphics[width=0.5\textwidth]{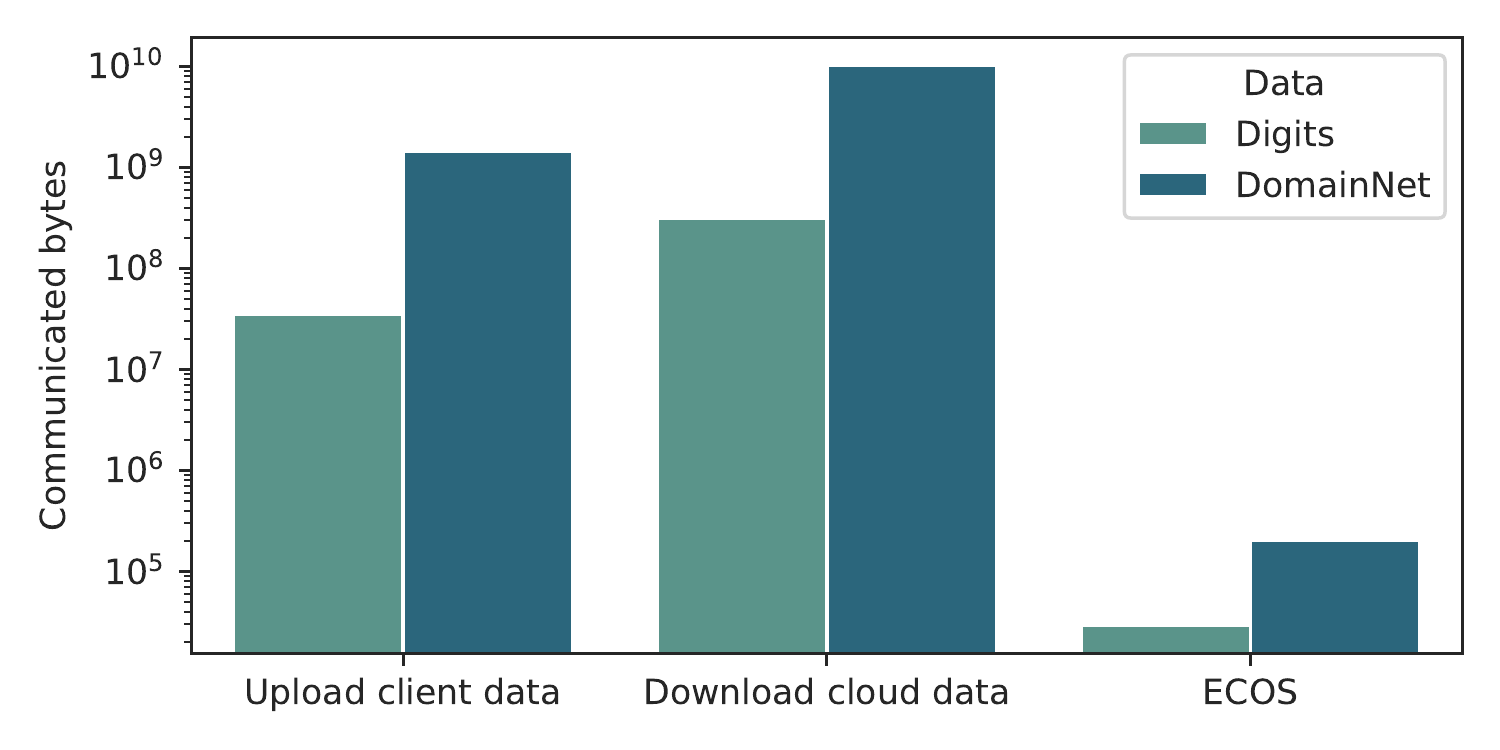}
    \caption{Evaluate the communication efficiency.}
    \label{fig:comm_eff}
\end{figure}

\textbf{Communication efficiency.}
We also compare the communication efficiency to the full cloud training (via uploading the whole client dataset) and fully client training (via downloading cloud dataset) in \cref{fig:comm_eff}.
For the ECOS, we let the size of the query set be $100$, which is the default configuration in our experiments.
Because the ECOS only communicates a few low-dimensional features (for example, 512-dimensional ResNet-extracted features for DomainNet and 72-dimensional HOG features for Digits), it costs much fewer bytes compared to traditional outsourcing by uploading the client data.
To be concrete, we also present the cost of downloading the cloud data and it is way more expensive than the rest two methods.

\end{document}